\documentclass[11pt,letter]{article} 

\usepackage[table,xcdraw,dvipsnames]{xcolor}
\usepackage{graphicx}

\usepackage{mathtools}
\usepackage{amsmath}
\usepackage{bbm}
\usepackage{amsfonts}
\usepackage{amssymb}
\usepackage{bm,comment,environ,grffile,latexsym,floatrow,verbatim,xspace,xpatch}

\usepackage[english]{babel}
\usepackage{amsthm}

\usepackage{natbib}
\bibpunct{(}{)}{;}{a}{,}{,}

\usepackage{algorithm}
\usepackage{algorithmic}

\usepackage{fullpage}
\usepackage[colorlinks,citecolor=blue,urlcolor=blue,linkcolor=blue]{hyperref}
\usepackage[utf8]{inputenc} 
\usepackage[T1]{fontenc}
\usepackage{tgbonum}
\usepackage{tgadventor}
\usepackage{lmodern}
\usepackage{placeins}

\usepackage{url}            
\usepackage{booktabs}       
\usepackage{nicefrac}       
\usepackage{microtype}      

\usepackage{color-edits}
\addauthor{sr}{red}

\input{preamble.tex}

\newcommand{\E}{\mathbb{E}}
\newcommand{\reals}{\mathbb{R}}
\newcommand{\pred}{\widehat{y}}

\newcommand{\ind}[1]{\bone\left\{ #1 \right\}}
\def\EE{{\mathbb{E}}}
\newcommand{\RegExp}{\mathsf{R}}
\newcommand{\RegEmp}{\widehat{\mathsf{R}}}
\newcommand{\xtm}{{\sf eXtreme}\xspace}
\newcommand{\igw}{{\sf IGW}\xspace}
\newcommand{\reg}{{\sf regression-oracle}\xspace}

\newcommand{\chosenset}{\mathcal{A}}
\newcommand{\feedbackset}{\Phi}


\setlist[compactitem]{leftmargin=*}

\usepackage{titlesec}
\titlespacing\section{0pt}{\parskip - 0em}{\parskip - 0.1em}
\titlespacing\subsection{0pt}{\parskip - 0.1em}{\parskip - 0.2em}

\newcommand\blfootnote[1]{%
  \begingroup
  \renewcommand\thefootnote{}\footnote{#1}%
  \addtocounter{footnote}{-1}%
  \endgroup
}

\usepackage{titlesec}
\title{Top-$k$ eXtreme Contextual Bandits with Arm Hierarchy}
\author{Rajat Sen$^1$ \and Alexander Rakhlin$^{2,3}$ \and Lexing Ying$^{4,3}$ \and Rahul Kidambi$^{3}$ \and Dean Foster$^{3}$ \and Daniel Hill$^{3}$ \and Inderjit Dhillon$^{5,3}$}

\begin{document}
\maketitle
\begin{abstract}
    Motivated by modern applications, such as online advertisement and recommender systems, we study the top-$k$ \xtm contextual bandits problem, where the total number of arms can be enormous, and the learner is allowed to select $k$ arms and observe all or some of the rewards for the chosen arms. We first propose an algorithm for the non-\xtm  realizable setting, utilizing the Inverse Gap Weighting strategy for selecting multiple arms. We show that our algorithm has a regret guarantee of $O(k\sqrt{(A-k+1)T \log (|\cF|T)})$, where $A$ is the total number of arms and $\cF$ is the class containing the regression function, while only requiring $\tilde{O}(A)$ computation per time step. In the $\xtm$ setting, where the total number of arms can be in the millions, we propose a practically-motivated arm hierarchy model that induces a certain structure in mean rewards to ensure statistical and computational efficiency. The hierarchical structure allows for an exponential reduction in the number of relevant arms for each context, thus resulting in a regret guarantee of $O(k\sqrt{(\log A-k+1)T \log (|\cF|T)})$. Finally, we implement our algorithm using a hierarchical linear function class and show superior performance with respect to well-known benchmarks on simulated bandit feedback experiments using \xtm multi-label classification datasets. On a dataset with three million arms, our reduction scheme has an average inference time of only 7.9 milliseconds, which is a 100x improvement.
\end{abstract}

\fontfamily{lmr}\selectfont

\section{Introduction}
\label{sec:intro}
\blfootnote{$^1$ Google Research, Work done while at Amazon. ~$^{2}$ MIT. ~$^{3}$ Amazon. ~$^{4}$  Stanford University. ~$^{5}$  University of Texas at Austin.}
The \textit{contextual bandit} is a sequential decision-making problem, in which, at every time step, the learner observes a context, chooses one of the $A$ possible actions (arms), and receives a reward for the chosen action. Over the past two decades, this problem has found a wide range of applications, from e-commerce and recommender systems~\citep{yue2011linear, li2016collaborative} to medical trials~\citep{durand2018contextual, villar2015multi}. The aim of the decision-maker is to minimize the difference in total expected reward collected when compared to an optimal policy, a quantity termed \textit{regret}. As an example, consider an advertisement engine in an online shopping store, where the context can be the user's query, the arms can be the set of millions of sponsored products and the reward can be a click or a purchase. In such a scenario, one must balance between \textit{exploitation} (choosing the best ad (arm) for a query (context) based on current knowledge) and \textit{exploration} (choosing a currently unexplored ad for the context to enable future learning).

The contextual bandits literature can be broadly divided into two categories. The \textit{agnostic} setting~\citep{agarwal2014taming, langford2007epoch, beygelzimer2011contextual, rakhlin2016bistro} is a model-free setting where one competes against the best policy (in terms of expected reward) in a class of policies. On the other hand, in the \textit{realizable} setting it is assumed that a known class $\cF$ contains the function mapping contexts to expected rewards. Most of the algorithms in the realizable setting are based on Upper Confidence Bound or Thompson sampling~\citep{filippi2010parametric, chu2011contextual, krause2011contextual, agrawal2013thompson} and require specific parametric assumptions on the function class. Recently there has been exciting progress on contextual bandits in the realizable case with general function classes. \cite{foster2020beyond} analyzed a simple algorithm for general function classes that reduced the adversarial contextual bandit problem to online regression, with a minimax optimal regret scaling. The algorithm was then analyzed for i.i.d. contexts using offline regression in~\citep{simchi2020bypassing}. The proposed algorithms are general and easily implementable but have two main shortcomings. 

First, in many practical settings the task actually involves selecting a small number of arms per time instance rather than a single arm. For instance, in our advertisement example, the website can have multiple slots to display ads and one can observe the clicks received from some or from all the slots. It is not immediately obvious how the techniques in~\citep{simchi2020bypassing, foster2020beyond} can be extended to selecting $k$ of a total of $A$ arms while avoiding the combinatorial explosion from $A \choose k$ possibilities. Second, the total number of arms $A$ can be in tens of millions and we need to develop algorithms that only require $o(A)$ computation per time-step and also have a much smaller dependence on the total number of arms in the regret bounds. Therefore, in this paper, we consider the top-$k$ \xtm  contextual bandit problem where the number of arms is potentially enormous and at each time-step one is allowed to select $k \geq 1$ arms.

This extreme setting is both theoretically and practically challenging, due to the sheer size of the arm space. On the theoretical side, most of the existing results on contextual bandit problems address the small arm space case, where the complexity and regret typically scales polynomially (linearly or as square root) in terms of the number of arms (with the notable exception of the case when arms are embedded in a $d$-dimensional vector space \citep{foster2020adapting}). Such a scaling inevitably results in large complexity and regret in the extreme setting. On the implementation side, most contextual bandit algorithms have not been shown to scale to millions of arms. The goal of this paper is to bridge the gaps both in theory and in practice. We show that the freedom to present more than one arm per time step provides valuable exploration opportunities. Moreover in many applications, for a given context, the rewards from the arms that are correlated to each other but not directly related to the context are often quite similar, while large variations in the reward values are only observed for the arms that are closely related to the context. For instance in the advertisement example, for an electronics query (context) there might be finer variation in rewards among computer accessories related display ads while very little variation in rewards among items in an unrelated category like culinary books. This prior knowledge about the structure of the reward function can be modeled via a judicious choice of the model class $\cF$, as we show in this paper.

The \textbf{main contributions} of this paper are as follows:
\begin{compactitem}
\item We define the top-$k$ contextual bandit problem in Section~\ref{sec:topk_cb}. We propose a natural modification of the inverse gap weighting (\igw) sampling strategy employed  in~\citep{foster2020beyond, simchi2020bypassing,abe1999associative} as Algorithm~\ref{alg:topk}. In Section~\ref{sec:topk_results} we show that our algorithm can achieve a top-$k$ regret bound of $O(k\sqrt{(A-k+1)T \log (|\cF|T)})$ where $T$ is the time-horizon. Even though the action space is combinatorial, our algorithm's computational cost for a time-step is $O(A)$ as it can leverage the additive structure in the total reward obtained from a set of arms chosen. We also prove that if the problem setting is only approximately realizable then our algorithm can achieve a regret scaling of $O(k\sqrt{(A-k+1)T \log (|\cF|T)} + \epsilon k\sqrt{A-k+1}T)$, where $\epsilon$ is a measure of the approximation.

\item Inspired by success of tree-based approached for \xtm output space problems in supervised learning~\citep{prabhu2018parabel, pecos2020, khandagale2020bonsai}, in Section~\ref{sec:xtm} we introduce a hierarchical structure on the set of arms to tackle the \xtm setting. This allows us to propose an \xtm reduction framework that reduces an extreme contextual bandit   problem with $A$ arms ($A$ can be in millions) to an equivalent problem with only $O(\log A)$ arms. Then we show that our regret guarantees from Section~\ref{sec:topk_results} carry over to this reduced problem.

\item We implement our \xtm contextual bandit algorithm with a hierarchical linear function class and test the performance of different exploration strategies under our framework on \xtm multi-label datasets~\citep{Bhatia16} in Section~\ref{sec:sims}, under simulated bandit feedback~\citep{bietti2018contextual}. On the amazon-3m dataset, with around three million arms, our reduction scheme leads to a 100x improvement in inference time over a naively evaluating the estimated reward for every arm given a context. We show that the \xtm reduction also leads to a 29\% improvement in progressive mean rewards collected on the eurlex-4k dataset.
More over we show that our exploration scheme has the highest win percentage among the 6 datasets w.r.t the baselines.
\end{compactitem}








\section{Related Work}
\label{sec:rwork}
The relevant prior work can be broadly classified under the following three categories:

{\bf General Contextual Bandits: } The general contextual bandit problem has been studied for more than two decades. In the agnostic setting where the mean reward of the arms given a context is not fully captured by the function class $\cF$, the problem was  studied in the adversarial setting leading to the well-known EXP-4 class of algorithms~\citep{auer2002nonstochastic, mcmahan2009tighter, beygelzimer2011contextual}. These algorithms can achieve the optimal $\tilde{O}(\sqrt{AT\log (T|\cF|)})$ regret bound but the computational cost per time-step can be $O(|\cF|)$. This paved the way for oracle-based contextual bandit algorithms in the stochastic setting~\citep{agarwal2014taming, langford2007epoch}. The algorithm in~\citep{agarwal2014taming} can achieve optimal regret bounds while making only $\tilde{O}(\sqrt{AT})$ calls to a cost-sensitive classification oracle, however the algorithm and the oracle are not easy to implement in practice. In more recent work, it has been shown that algorithms that use regression oracles work better in practice~\citep{foster2018practical}. In this paper we will be focused on the realizable (or near-realizable) setting, where there exists a function in the function class, which can model the expected reward of arms given context. This setting has been studied with great practical success under specific instances of the function classes, such as linear. Most of the successful approaches are based on Upper Confidence Bound strategies or Thompson Sampling~\citep{filippi2010parametric, chu2011contextual, krause2011contextual, agrawal2013thompson}, both of which lead to algorithms which are heavily tailored to the specific function class. The general realizable case was modeled in~\citep{agarwal2012contextual} and recently there has been exciting progress in this direction. The authors in~\citep{foster2020beyond} identified that a particular exploration scheme that dates back to~\citep{abe1999associative} can lead to a simple algorithm that reduces the contextual bandit problem to online regression and can achieve optimal regret guarantees. The same idea was extended for the stochastic realizable contextual bandit problem with an offline batch regression oracle~\citep{simchi2020bypassing,foster2020instance}. We build on the techniques introduced in these works. However all the literature discussed so far only address the problem of selecting one arm per time-step, while we are interested in selection the top-$k$ arms at each time step. 

{\bf Exploration in Combinatorial Action Spaces: } In~\citep{qin2014contextual} authors study the $k$-arm selection problem in contextual bandits where the function class is linear and the utility of a set of arms chosen is a set function with some monotonicity and Lipschitz continuity properties. In~\citep{yue2011linear} the authors study the problem of retrieving $k$-arms in contextual bandits in the context of a linear function class and the assumption that the utility of a set of arms is sub-modular. Both these approaches do not extend to general function classes and are not applicable to the extreme setting. In the context of off-policy learning from logged data there are several works that address the top-$k$ arms selection problem under the context of slate recommendations~\citep{swaminathan2017off, narita2019efficient}. We will now review the combinatorial action space literature in multi-armed bandit (MAB) problems. Most of the work in this space deals with semi-bandit feedback~\citep{chen2016combinatorial, combes2015combinatorial, kveton2015tight, merlis2019batch}. This is also our feedback model, but we work in a contextual setting. There is also work in the full-bandit feedback setting, where one gets to observe only one representative reward for the whole set of arms chosen. This body of literature can be divided into the adversarial setting~\citep{merlis2019batch, cesa2012combinatorial} and the stochastic setting~\citep{dani2008stochastic, agarwal2018regret, lin2014combinatorial, rejwan2020top}.

{\bf Learning in \xtm Output Spaces: } The problem of learning from logged bandit feedback when the number of arms is extreme was studied recently in~\citep{lopez2020learning}. In~\citep{majzoubi2020efficient} the authors address the contextual bandit problem for continuous action spaces by using a cost sensitive classification oracle for large number of classes, which is itself implemented as a hierarchical tree of binary classifiers. In the context of supervised learning the problem of learning under large but correlated output spaces has been studied under the banner of \xtm Multi-Label Classification/Ranking (XMC/ XMR) (see~\citep{Bhatia16} and references). Tree based methods for XMR have been extremely successful~\citep{jasinska2016extreme, prabhu2018parabel, khandagale2020bonsai, wydmuch2018no, you2019attentionxml, pecos2020}. In particular our assumptions about arm hierarchy and the implementation of our algorithms have been motivated by~\citep{prabhu2018parabel, pecos2020}. 


\section{Top-$k$ Stochastic Contextual Bandit Under Realizability}
\label{sec:topk}
In the standard contextual bandit problem, at each round, a context is revealed to the learner, the learner picks a single arm, and the reward for only that arm is revealed. In this section, we will study the top-$k$ version of this problem, i.e. at each round the learner selects $k$ distinct arms, and the total reward corresponds to the sum of the rewards for the subset. As feedback, the learner observes some of the rewards for actions in the chosen subset, and we allow this feedback to be as rich as the rewards for all the $k$ selected arms or as scarce as no feedback at all on the given round. 


\subsection{The Top-$k$ Problem} 
\label{sec:topk_cb}

Suppose that at each time step $t \in \{1,\ldots,T\}$, the environment generates a context $x_t \in \Xcal$ and rewards  $\{r_t(a)\}_{a \in [A]}$ for the $A$ arms. The set of arms will be denoted by $\cA = [A] := \{1, 2, \cdots, A\}$. As standard in the stochastic model of contextual bandits, we shall assume that $(x_t, r_t(1), \cdots r_t(A))$ are generated i.i.d. from a fixed but unknown distribution $\Dcal$ at each time step.  In this work we will assume for simplicity that $r_t(a) \in [0,1]$ almost surely for all $t$ and $a \in [A]$. We will work under the realizability assumption~\citep{agarwal2012contextual, foster2018practical, foster2020beyond}. We also provide some results under approximate realizabilty or the misspecified setting similar to~\citep{foster2020beyond}.

\begin{assumption} [{\bf Realizability}]
	\label{asum:realizability}  There exists an $f^* \in \Fcal$ such that, 
	\begin{align}
	  \EE[r_{t}(a) \vert X = x] = f^*(x, a) ~~\forall x \in \Xcal, a \in [A],
	\end{align}
	where $\cF$ is a class of functions $\Xcal \times \Acal \rightarrow [0, 1]$ known to the decision-maker.
\end{assumption}


\begin{assumption} [{\bf $\epsilon$-Realizability}]
	\label{asum:realizability2}  There exists an $f^* \in \Fcal$ such that, 
	\begin{align}
	  \lvert \EE[r_{t}(a) \vert X = x] - f^*(x, a) \rvert \leq \epsilon ~~~~~~~~ \forall x \in \Xcal, a \in [A].
	\end{align}
	where $\cF$ is a class of functions $\Xcal \times \Acal \rightarrow [0, 1]$ known to the decision-maker.
\end{assumption}

We assume that the misspecification level $\epsilon$ is known to the learner and refer to \citep{foster2020adapting} for techniques on adapting to this parameter.

{\bf Feedback Model and Regret.~} At the beginning of the time step $t$, the learner observes the context $x_t$ and then chooses a set of $k$ \textit{distinct} arms $\chosenset_t\subseteq \cA$, $|\chosenset_t|=k$. 
The learner receives feedback for a subset $\feedbackset_t\subseteq \chosenset_t$, that is, $r_t(a)$ is revealed to the learner for every $a\in\feedbackset_t$.
\begin{assumption}
\label{asum:feedback}
    Conditionally on $x_t,\chosenset_t$ and the history $\cH_{t-1}$ up to time $t-1$, the set $\feedbackset_t\subseteq\chosenset_t$ is random and for any $a\in \chosenset_t$,
    \[\mathbb{P}(a\in \Phi_t|x_t,\chosenset_t, \cH_{t-1})\geq c \]
    for some $c\in (0,1]$ which we assume to be known to the learner.
\end{assumption}
For the advertisement example, Assumption~\ref{asum:feedback} means that the user providing feedback has at least some non-zero probability $c>0$ of choosing each of the presented ads, marginally. The choice $c=1$ corresponds to the most informative case -- the learner receives feedback for all the $k$ chosen arms. On the other hand, for $c<1$ it may happen that no feedback is given on a particular round (for instance, if $\feedbackset_t$ includes each $a\in\chosenset_t$ independently with probability $c$). When $\chosenset_t$ is a ranked list, behavioral models postulate that the user clicks on an advertisement according to a certain distribution with decreasing probabilities; in this case, $c$ would correspond to the smallest of these probabilities. A more refined analysis of regret bounds in terms of the distribution of $\Phi_t$ is beyond the scope of this work.

The total reward obtained in time step $t$ is given by the sum $\sum_{a\in\chosenset_t} r_t(a)$ of all the individual arm rewards in the chosen set, regardless of whether only some of these rewards are revealed to the learner. 
The performance of the learning algorithm will be measured in terms of \textit{regret}, which is the difference in mean rewards obtained as compared to an optimal policy which always selects the top $k$ distinct actions with the highest mean reward. To this end, let $\chosenset_t^*$ be the set of $k$ distinct actions that maximizes $\sum_{a\in \chosenset_t^*} f(x_t,a)$ for the given $x_t$. Then the expected regret is
\begin{equation}
  R(T) := \sum_{t = 1}^{T} \EE \left[\sum_{a\in \chosenset_t^*} f^*(x_t, a) - \sum_{a\in\chosenset_t} f^*(x_t, a) \right].
\end{equation}

{\bf Regression Oracle.} As in \citep{foster2018practical,simchi2020bypassing}, we will rely on the availability of an optimization oracle \reg for the class $\cF$ that can perform least-squares regression,
\begin{equation}
\label{eq:oracle} \argmin_{f \in \Fcal}\sum_{s=1}^{t} (f(x_a, a_s) - r_s)^2 
\end{equation}
where $(x, a, r) \in \Xcal \times \cA \times [0,1]$ ranges over the collected data.

\subsection{\igw for top-$k$ Contextual Bandits}
\label{sec:topk_igw}

Our proposed algorithm for top-$k$ arm selection in general contextual bandits in a non-extreme setting is provided as Algorithm~\ref{alg:topk}. It is a natural extension of the Inverse Gap Weighting (\igw) sampling scheme~\citep{abe1999associative, foster2020beyond, simchi2020bypassing}. In Section~\ref{sec:topk_results} we will show that this algorithm with $r=1$ has good regret guarantees for the top-$k$ problem even though the action space is combinatorial, thanks to the linearity of the regret objective in terms of rewards of individual arms in the subset. Note that a naive extension of \igw by treating each action in $\cA^{k}$ as a separate arm would require a computation of $O\left( A \choose k \right)$ per time step and a similar regret scaling. In contrast, Algorithm~\ref{alg:topk} only requires $\tilde{O}(A)$ computation for the sampling per time step.

\begin{algorithm}[!ht]
\small
\caption{Top-$k$ Contextual Bandits with \igw}
\label{alg:topk}
\begin{algorithmic}[1]
\STATE {\bf Arguments:~} $k$ and $r$ (number of explore slots, $ 1 \leq r \leq k $) 
 \FOR{$l \gets 1$ \textbf{to} $e(T)$} 
   \STATE Fit regression oracle to all past data
   \STATE $\pred_l = \argmin_{f\in\cF} \sum_{t=1}^{N_{l-1}} \sum_{a\in \feedbackset_t} (f(x_{t}, a)-r_{t}(a))^2$
  \FOR {$s \gets N_{l-1}+1$ \textbf{to} $N_{l-1} + n_{l}$}
    \STATE Receive $x_s$
     \STATE Let $\widehat{a}_s^1,\ldots,\widehat{a}_s^A$ be the arms ordered in decreasing order according to $\pred_l(x_s,\cdot)$ values.
     \STATE $\chosenset_s = \{\widehat{a}_s^1, \cdots, \widehat{a}_s^{k-r}\}$.
     \FOR {$c \gets 1$ \textbf{to} $r$}
        \STATE Compute randomization distribution 
        \STATE $p= \igw\left(\{\cA \setminus \chosenset_s\}; \pred_l(x_s, \cdot)\right)$. 
        \STATE Sample $a\sim p$. Let $\chosenset_s=\chosenset_s \cup \{a\}$.
    \ENDFOR
    \STATE Obtain rewards $r_{s}(a)$ for actions $a\in \feedbackset_s\subseteq \chosenset_s$.
    \ENDFOR
  \STATE Let $N_{l} = N_{l-1} + n_l$
  \ENDFOR
\end{algorithmic}
\end{algorithm}
The Inverse Gap Weighting strategy was introduced in~\citep{abe1999associative} and has since then been used for contextual bandits in the realizable setting with general function classes~\citep{foster2020beyond, simchi2020bypassing, foster2020instance}. Given a set of arms $\cA$, an estimate $\pred:\cX\times\cA\to\reals$ of the reward function, and a context $x$, the distribution $p= \igw\left(\cA; \pred(x, \cdot)\right)$ over arms is given by
\begin{align*}
    p(a | x) = \begin{cases} 
        \frac{1}{|\cA| + \gamma_l(\hat{y}(x, a_{\star}) - \hat{y}(x, a))} & \mbox{if } a \neq a_{\star} \\
        1 - \sum_{a' \in \cA: a' \neq a_{\star}} p(a'|x) & \mbox{otherwise} 
    \end{cases}
\end{align*}
where $a_{\star} = \argmax_{a \in \cA}\pred(x, a)$, $\gamma_l$ is a scaling factor. 
 
Algorithm~\ref{alg:topk} proceeds in epochs, indexed by $l=1,\ldots,e(T)$. Note that $N_{e(T)} = \sum_{l=1}^{e(T)} n_l = T$. The regression model is updated at the beginning of the epoch with all the past data and used throughout the epoch ($n_l$ time steps). The arm selection procedure for the top-$k$ problem involves selecting the top ($k-r$) arms \textit{greedily} according to the current estimate $\pred_l$ and then selecting the rest of the arms at random according to the Inverse Gap Weighted distribution over the set of remaining arms. For $r>1$, the distribution is recomputed over the remaining support every time an arm is selected.

\subsection{Regret of \igw for top-$k$ Contextual Bandits}
\label{sec:topk_results}
In this section we show that our algorithm has favorable regret guarantees. Our regret guarantees are only derived for the case when Algorithm~\ref{alg:topk} is run with $r=1$. However, we will see that other values of $r$ also work well in practice in Section~\ref{sec:sims}. For ease of exposition we assume $\cF$ is finite; our results can be extended to infinite function classes with standard techniques (see e.g. ~\citep{simchi2020bypassing}). We first present the bounds under exact realizability.\footnote{We have not optimized the constants in the definition of $\gamma_l$.}

\begin{theorem}
\label{thm:topkregret} Algorithm~\ref{alg:topk} under Assumptions~\ref{asum:realizability} and \ref{asum:feedback}, when run with parameters
\begin{align*}
    &r = 1;~~~~ N_l = 2^l; ~~~~ \gamma_l = \frac{1}{32}\sqrt{\frac{c(A-k+1)N_{l-1}}{162\log \left( \frac{|\cF| T^3}{\delta}\right)}},
\end{align*}
has regret bound
\begin{align*}
    R(T) = \cO \left(k\sqrt{c^{-1}(A - k +1) T  \log \left( \frac{|\cF| T}{\delta} \right)} \right)
\end{align*}
with probability at least $1 - \delta$, for a finite function class $\cF$.
\end{theorem}

In the next theorem we bound the regret under $\epsilon$-realizability.

\begin{theorem}
\label{thm:topkregret2} Algorithm~\ref{alg:topk} under Assumptions~\ref{asum:realizability2} and \ref{asum:feedback}, when run with parameters
\begin{align*}
    &r = 1;~~~~ N_l = 2^l; ~~~~ \gamma_l = \frac{\sqrt{c(A - k + 1)}}{32 \sqrt{\frac{420}{N_{l-1}}
      \log \left( \frac{|\cF| T^3}{\delta}\right) + 2 \epsilon^2}}
\end{align*}
has regret bound
\begin{align*}
    R(T) &= \cO \left(k\sqrt{c^{-1}(A - k +1) T  \log \left( \frac{|\cF| T}{\delta} \right)}+ \epsilon kT\sqrt{A - k + 1}\right)
\end{align*}
with probability at least $1 - \delta$, for a finite function class $\cF$.
\end{theorem}

The proofs for both of our main theorems are provided in Appendix~\ref{sec:topkanalysis}. One of the key ingredients in the proof is an induction hypothesis which helps us relate the top-$k$ regret of a policy with respect to the estimated reward function $\pred_l \in \cF$ at the beginning of epoch $l$ to the actual regret with respect to $f^* \in \cF$. The argument can be seen as a generalization of \citep{simchi2020bypassing} to $k>1$. 


\section{\xtm Contextual Bandits and Arm Hierarchy} 
\label{sec:xtm}

When the number of arms $A$ is large, the goal is to design algorithms so that the computational cost per round is poly-logarithmic in $A$ (i.e. $\Ocal(\mathrm{polylog}(A))$) and so it the overall regret. However,
owing to known lower bounds~\citep{foster2020beyond}, this cannot be achieved without imposing
further assumptions on the contextual bandit problem.


{\bf Main idea.}  
A key observation is that the regression-oracle framework does not impose any restriction on the structure of the arms and in fact the set of arms can even be context-dependent. Let us assume that,
\begin{compactitem}
\item For each $x$, there is an $x$-dependent decomposition 
  \begin{equation}\label{eqn:xpart}
    \cA_x := \{ \ba_{x,1}, \cdots, \ba_{x,Z} \},
  \end{equation}
  where $\ba_{x,1}, \cdots, \ba_{x,Z}$ form a disjoint union of $\cA$ with $Z = \Ocal(\log A)$.
\item For any two arms $a$ and $a'$ from any subset $\ba_{x,i}$, the expected reward function
  $r(x,a)=\EE[r(a)|X=x]$ satisfies the following consistency condition
  \begin{equation}\label{eqn:consistent}
    |r(x,a) - r(x,a')| \le \epsilon.
  \end{equation}
\end{compactitem}
By treating $\ba_{x,1},\cdots,\ba_{x,Z}$ as effective arms, the results of Section
\ref{sec:topk_results} can be applied by {\em working with functions that are piecewise constant over each $\ba_{x,i}$}.  Such a context-dependent arm space decomposition is a reasonable assumption, because often the rewards from a large subset of arms exhibit minor variations for a given context $x$.

{\bf Motivating example.} To motivate and justify the conditions \eqref{eqn:xpart} and
\eqref{eqn:consistent}, consider a simple but representative setting where the contexts in $\Xcal$
and arms in $\cA$ are both represented as feature vectors in $\RR^d$ for a fixed dimension $d$ and
the distance between two vectors is measured by the Euclidean norm $\|\cdot\|$. In many
applications, the expected reward $r(x,a)$ satisfies the gradient condition $ |\partial_a r(x,a)| \le \frac{\eta}{\|x-a\|}$,
for some $\eta>0$, i.e., $r(x,a)$ is sensitive in $a$ only when $a$ is close to $x$ and
insensitive when $a$ is far away from $x$. 



Let us introduce a hierarchical decomposition $\Tcal$ for $\cA$, which in this case is a balanced
$2^d$-ary tree. At the leaf level, each tree node has a maximum number of $m$ arms from the extreme
arm space $\cA$. The height of such a tree is $H\approx\ceil{\log\ceil{A/m}}$ under some mild
assumptions on the distributions of the arms in $\cA$. For a specific depth $h$, we use $e_{h,i}$ to
denote a node with index $i$ at depth $h$ and $\cC_{h,i}$ to denote the $2^d$ children of $e_{h,i}$
at depth $h+1$. Each node $e_{h,i}$ of the tree is further equipped with a \textit{routing function}
$g_{h,i}(x)=\frac{\mathrm{rad}_{h,i}}{\|x-\mathrm{ctr}_{h,i}\|}$, where $\mathrm{ctr}_{h,i}$ is the
center of the node $e_{h,i}$ and $\mathrm{rad}_{h,i}$ is the radius of the smallest ball at
$\mathrm{ctr}_{h,i}$ that contains $e_{h,i}$. The center $\mathrm{ctr}_{h,i}$ serves as a
representative for the set of arms in $e_{h,i}$. Figure \ref{fig:hier} (left) illustrates the
hierarchical decomposition for the 1D case.

Given a context $x$, we perform an \textit{adaptive search} through this hierarchical decomposition
$\Tcal$, parameterized by a constant $\beta\in(0,1)$. Initially, the sets $I_x$ and $S_x$ are set to
be empty and the search starts from the root of the tree. When a node $e_{h,i}$ is visited, it is
considered {\em far from $x$} if
$g_{h,i}(x)=\frac{\mathrm{rad}_{h,i}}{\|x-\mathrm{ctr}_{h,i}\|}\le\beta$ and {\em close to $x$} if
$g_{h,i}(x)=\frac{\mathrm{rad}_{h,i}}{\|x-\mathrm{ctr}_{h,i}\|}>\beta$. If $e_{h,i}$ is far from
$x$, we simply place it in $I_x$. If $e_{h,i}$ close to $x$, we visit its children in $\cC_{h,i}$
recursively if $e_{h,i}$ is an internal node or place it in $S_x$ if it is a leaf. At the end of the
search, $I_x$ consists of a list of internal nodes and $S_x$ is a list of {\em singleton} arms.

We claim that the union of the singleton arms in $S_x$ and the nodes in $I_x$ form an
$x$-dependent decomposition $\cA_x$. First, the disjoint union of $I_x$ and $S_x$ covers the whole
arm space $\cA$. $S_x$ contains only $O(1)$ singleton arms with arm features close to the context
feature $x$ while the size of $I_x$ is bounded by $O(\log A)$ as there are at most $O(1)$ nodes
$e_{h,i}$ inserted into $I_x$ at each of the $O(\log A)$ levels. Hence, the sum of the cardinalities
of $S_x$ and $I_x$ is bounded by $Z=O(\log A)$, i.e., logarithmic in the size $A$ of the extreme arm
space $\cA$.

Second, for any two original arms $a_1,a_2$ corresponding to a node $e_{h,i} \in I_x$,
\[
|r(x,a_1)-r(x,a_2)| \le \|\partial_l r(x,a')\| \cdot \|a_1-a_2\| 
\le \frac{\eta}{\|x-a'\|}\cdot (2\ \mathrm{rad}_{h,i}),
\]
where $a'$ lies on the segment between $a_1$ and $a_2$. Since
\[
\|x-a'\|\ge\|x-\mathrm{ctr}_{h,i}\|-\|a'-\mathrm{ctr}_{h,i}\| \ge (1/\beta-1) \mathrm{rad}_{h,i}
\]
holds for $e_{h,i} \in I_x$,
\[
|r(x,a_1)-r(x,a_2)| \le \frac{\eta}{(1/\beta-1)\mathrm{rad}_{h,i}} \cdot (2\ \mathrm{rad}_{h,i}) = \frac{2\eta \beta}{1-\beta}.
\]
Hence, if one chooses $\beta$ so that $2\eta\beta/(1-\beta) \le \epsilon$, then
$|r(x,a_1)-r(x,a_2)|\le\epsilon$ for any two arms $a_1,a_2$ in any $e_{h,i}\in I_x$.

Therefore for each $x$, the union of the singleton arms in $S_x$ and the nodes in $I_x$ form an
$x$-dependent decomposition of $\cA$ that satisfies the
conditions \eqref{eqn:xpart} and \eqref{eqn:consistent}. Figure \ref{fig:hier} (middle) shows the
decomposition for a given context $x$, while Figure \ref{fig:hier} (right) shows how the
decomposition varies with the context $x$. In what follows, we shall refer to the members of $I_x$
{\em node effective arms} and the ones of $S_x$ {\em singleton effective arms}.

\begin{figure}[ht!]
  \centering
  \includegraphics[scale=0.4]{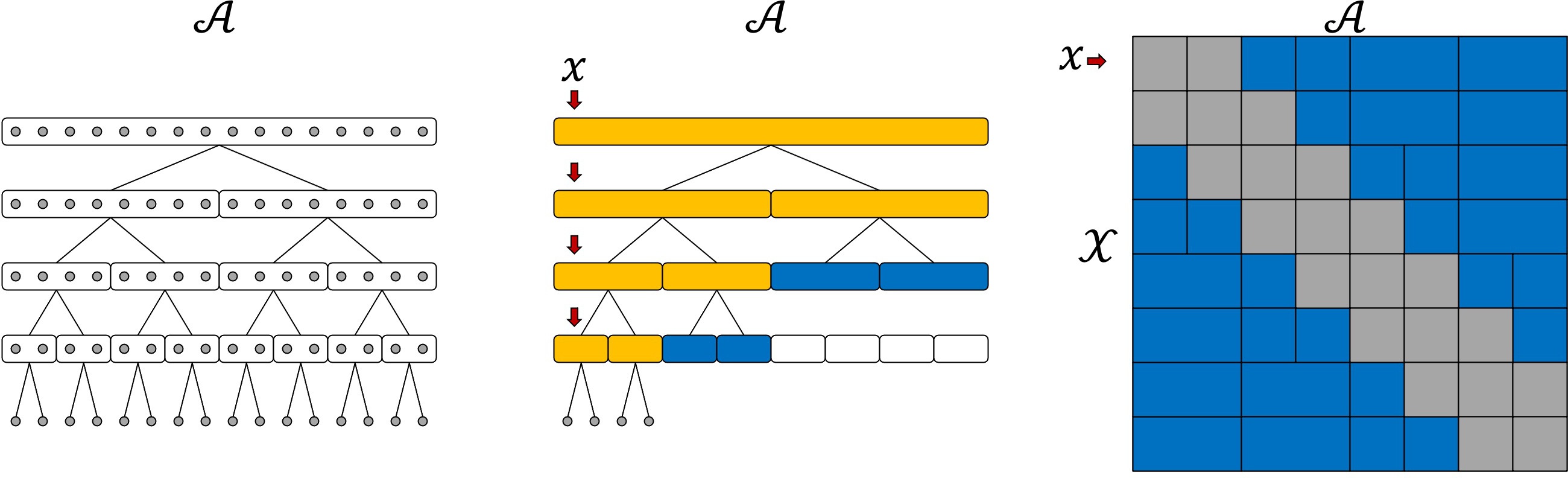}
  \caption{\small {\bf Left}: an illustration of the hierarchical decomposition for $\cA$, where each gray dot indicates an arm. {\bf Middle}: the adaptive search for a given context $x$. The yellow nodes are further explored as they are close to $x$ while the blue nodes are not as they are far from $x$. The set of effective arms for $x$ consists of the blue nodes and the singleton arms in the yellow leaf nodes.  {\bf Right}: For a fixed $x$, the corresponding row shows the $x$-dependent hierarchical arm space decomposition. As $x$ varies, the decomposition also changes. Each blue block stands for a non-singleton effective arm, valid for a contiguous block of contexts. Each gray block contains the singleton effective arms, valid again for a contiguous block of contexts. }
  \label{fig:hier}
\end{figure}

{\bf General setting.} Based on the motivating example, we propose an arm hierarchy for general
$\Xcal$ and $\cA$. We assume access to a hierarchical partitioning $\Tcal$ of $\cA$ that breaks
progressively into finer subgroups of similar arms. The partitioning can be represented by a balanced
tree that is $p$-ary till the leaf level. At the leaf level, each node can have a maximum of $m>p$
children, each of which is a singleton arm in $\cA$. The height of such a tree is
$H=\ceil{\log_p\ceil{A/m}}$. With a slight abuse of notation, we use $e_{h,i}$ to denote a node in the tree as well as the subset of singleton arms in the subtree of the node.

Each internal node $e_{h,i}$ is assumed to be associated with a routing function $g_{h,i}(x)$ mapping $\cX\rightarrow [0, 1]$ and $\cC_{h,i}$ is used to denote the immediate children of node $e_{h,i}$.  Based on these routing functions and an integer parameter $b$, we define a
beam search in Algorithm~\ref{alg:beam} for any context $x\in \cX$ as an input. During its
execution, this beam search keeps at each level $h$ only the top $b$ nodes that return the highest
$g_{h,i}(x)$ values. The output of the beam search, denoted also by $\cA_x$, is the union of a set of
nodes denoted as $I_x$ and a set of singleton arms denoted as $S_x$.  The tree structure ensures
that there are at most $bm$ singleton arms in $S_x$ and at most $(p-1)b(H-1)$ nodes in
$I_x$. Therefore, $\lvert \cA_x\rvert\leq(p-1)b(H-1)+bm = O(\log A)$, implying that $\cA_x$
satisfies \eqref{eqn:xpart}. Though the cardinality $\lvert\cA_x\rvert$ can vary slightly
depending on the context $x$, in what follows we make the simplifying assumption that
$\lvert\cA_x\rvert$ is equal to a constant $Z=O(\log A)$ independent of $x$ and denote $\cA_x=\{\ba_{x,1},\ldots, \ba_{x,Z}\}$.

\begin{algorithm}[H]
\small
  \caption{Beam search}
  \label{alg:beam}
  \begin{algorithmic}[1]
    \STATE {\bf Arguments}:  beam-size $b$, $\Tcal$, routing functions $\{g\}$, $x$
    \STATE Initialize $\texttt{codes} = [(1,1)]$ and $I_x^{b} = \emptyset$. 
    \FOR {$h = 1, \cdots, H - 1$}
    \STATE Let $\texttt{labels} = \cup_{(h-1, i) \in \texttt{codes}} \Ccal_{h-1,i}$. 
    \STATE Let $\texttt{codes}$ be  top-$b$ nodes in $\texttt{labels}$ according to the values $g_{h,i}(x)$. 
    \STATE  Add the nodes in $\texttt{labels}  \setminus \texttt{codes}$ to $I_x$.
    \ENDFOR
    \STATE  Let $S_x = \cup_{(H-1, i) \in \texttt{codes}} \Ccal_{H-1,i}$.
    \STATE Return $\cA_x = S_x \cup I_x$. 
  \end{algorithmic}
\end{algorithm}

To ensure the consistency condition \eqref{eqn:consistent} in the general case, one requires the expected reward function $r(x,a)$ to be nearly constant over each effective arm $\ba_{x,i}$ and
work with a function class that is constant over each $\ba_{x,i}$. The following definition formalizes this.

\begin{definition}
  \label{def:consistent}
  Given a hierarchy $\Tcal$ with routing function family $\{g_{h,i}(\cdot)\}$ and a beam-width $b$,
  a function $f(x,a)$ is $(\Tcal,g,b)$-constant if for every $ x \in \Xcal$
  \begin{align*}
    f(x, a) = f(x, a') \quad\text{for all}\quad a ,a' \in e_{h,i},
  \end{align*}
  for any node $e_{h,i}$ in $I_x \subset \cA_x$.  A class of functions $\Fcal$ is
  $(\Tcal,g,b)$-constant if each $f\in\Fcal$ is $(\Tcal,g,b)$-constant.
\end{definition}

Figure \ref{fig:consistent} (left) provides an illustration of a $(\Tcal,g,b)$-constant predictor
function for the simple case $\Xcal\subset[0,1]$ and $\cA\subset[0,1]$. In the \xtm setting, we always assume that our function class $\Fcal$ is $(\Tcal,g,b)$-constant. By further assuming that the expected
reward $r(x,a)$ satisfies either Assumption~\ref{asum:realizability} or
Assumption~\ref{asum:realizability2}, Condition~\eqref{eqn:consistent} is satisfied.

\begin{figure}[ht!]
  \centering
  \includegraphics[scale=0.40]{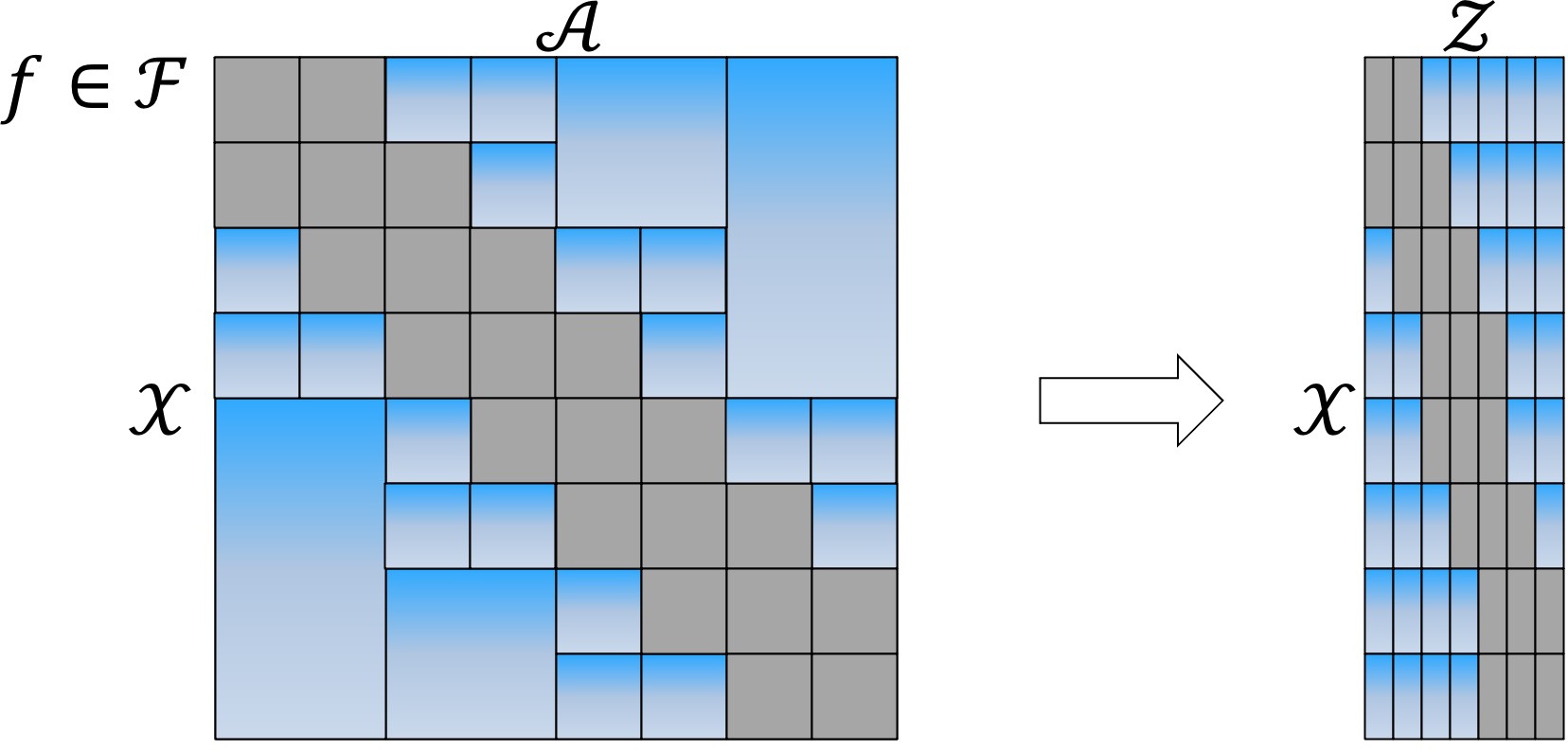}
  \caption{\small {\bf Left}: A $(\Tcal,g,b)$-constant predictor function $f(x,a)$ in the 1D motivating example with $\Xcal\subset[0,1]$ and $\cA\subset[0,1]$. Within each blue block, $f(x,a)$ is constant in $a$ but varies with $x$. {\bf Right}: the function $\ftil$ after the reduction.
  }
  \label{fig:consistent}
\end{figure}
  




\subsection{\igw for top-$k$ \xtm Contextual Bandits}
\label{sec:algo2}
In this section we provide our algorithm for the \xtm setting. As Definition~\ref{def:consistent} reduces the \xtm problem with $A$ arms to a non-extreme problem with only
$Z=\cO(\log A)$ effective arms, Algorithm~\ref{alg:hier} essentially uses the beam-search method in
Algorithm~\ref{alg:beam} to construct this reduced problem. The \igw randomization is performed over the
effective arms and if a non-singleton arm (i.e., an internal node of $\cT$) is chosen, we
substitute it with a randomly chosen singleton arm that lies in the sub-tree of that node.  More specifically, for a $(\Tcal,g,b)$-constant class $\Fcal$, we
define for each $f\in \Fcal$ a new function $\ftil: \cX \times [Z] \rightarrow [0,1]$ s.t.
for any $z=1,\ldots, Z$ we have $\ftil(x,z) = f(x,a) \;\text{for some fixed}\; a\in \ba_{x,z}$.
Here, we assume that for any $x$ the beam-search process in Algorithm~\ref{alg:beam} returns the
effective arms in $\cA_x$ in a fixed order and $\ba_{x,z}$ is the $z$-th arm in this order. The
collection of these new functions over the context set $\Xcal$ and the reduced arm space $\Zcal=[Z]$
is denoted by $\cFtil = \{\ftil: f\in\Fcal\}$. Figure \ref{fig:consistent} (right) provides an
illustration of a function $\ftil(x,z)$ obtained after the reduction.



\begin{algorithm}[!ht]
\small
\caption{\xtm Top-$k$ Contextual Bandits with \igw}
\label{alg:hier}
\begin{algorithmic}[1]
\STATE {\bf Arguments:~} $k$, number of explore slots: $ 1 \leq r \leq k $ 
  \FOR{$l \gets 1$ \textbf{to} $e(T)$} 
    \STATE Fit regression oracle to all past data
    \STATE $\pred_l = \argmin_{\ftil\in\cFtil} \sum_{t=1}^{N_{l-1}} \sum_{z\in \feedbackset_t} (\ftil(x_{t}, z)-\tilde{r}_{t}(z))^2$
    \FOR {$s \gets N_{l-1}+1$ \textbf{to} $N_{l-1} + n_{l}$}
      \STATE Receive $x_s$
      \STATE Use Algorithm~\ref{alg:beam} to get $\cA_{x_s}=\{\ba_{x_s,1},\ldots,\ba_{x_s,Z}\}$.
      \STATE Let $z_1,\ldots,z_Z$ be the arms in $[Z]$ in the descending order according to $\pred_l$.
      \STATE $\cZ_s = \{z_1, \cdots, z_{k-r}\}$.
      \FOR {$c \gets 1$ \textbf{to} $r$}
        \STATE Compute randomization distribution 
        \STATE $p= \igw\left( [Z] \setminus \cZ_s; \hat{y}_l(x_s, \cdot)\right)$. 
        \STATE Sample $z\sim p$. Let $\cZ_s = \cZ_s \cup \{z\}$.
      \ENDFOR
      \STATE $B_s = \{\}$.
      \FOR {$z$ in $\mathcal{Z}_s$}
        \STATE If $\ba_{x_s,z}$ is singleton arm, then add it to $B_s$.
        \STATE Otherwise sample a singleton arm $a$ in the subtree rooted at the node $\ba_{x_s,z}$ and add $a$ to $B_s$.
        \ENDFOR
      \STATE Choose the arms in $B_s$.
      \STATE Map the rewards back to the corresponding effective arms in $\mathcal{Z}_s$ and record $\{\tilde{r}_s(z), z\in \feedbackset_s\}$.
    \ENDFOR
  \STATE Let $N_{l} = N_{l-1} + n_l$
  \ENDFOR
\end{algorithmic}
\end{algorithm}

As a practical example, we can maintain the function class $\Fcal$ such that each member $f\in\Fcal$ is represented as a set of regressors at the internal nodes as well as the singleton arms in the tree. These regressors map contexts to $[0,1]$. For an $f\in\Fcal$, the regressor at each node is constant over the arms $a$ within this node and is only trained on past samples for which that node was selected
as a whole in $\mathcal{Z}_s$ in Algorithm~\ref{alg:hier}; the regressor at a
singleton arm can be trained on all samples obtained by choosing that arm. Note that even though we
might have to maintain a lot of regression functions, many of them can be sparse if the input
contexts are sparse, because they are only trained on a small fraction of past training samples.

\subsection{Top-$k$ Analysis in the \xtm Setting}
\label{sec:hier_analysis}
We can analyze Algorithm~\ref{alg:hier} under the realizability assumptions
(Assumption~\ref{asum:realizability} or Assumption~\ref{asum:realizability2}) when the class of
functions satisfies Definition~\ref{def:consistent}). Our main result is a reduction style argument
that provides the following corollary of Theorems~\ref{thm:topkregret} and~\ref{thm:topkregret2}.

\begin{corollary}
\label{cor:xtm}
Algorithm~\ref{alg:hier} when run with parameter $r=1$ has the following regret guarantees:

{\it (i)} If Assumptions~\ref{asum:realizability} and \ref{asum:feedback} hold and the function class $\cF$ is $(\cT, g, b)$-constant (Definition~\ref{def:consistent}), then setting parameters as in Theorem~\ref{thm:topkregret} ensures that the regret bound stated in Theorem~\ref{thm:topkregret} holds 
with $A$ replaced by $O(\log A)$.

{\it (ii)} If Assumptions~\ref{asum:realizability2} and \ref{asum:feedback} hold and the function class $\cF$ is $(\cT, g, b)$-constant (Definition~\ref{def:consistent}), then setting parameters as in Theorem~\ref{thm:topkregret2} ensures that the regret bound stated in Theorem~\ref{thm:topkregret2} holds with $A$ replaced by $O(\log A)$.


\end{corollary}


\section{Empirical Results}
\label{sec:sims}
We compare our algorithm with well known baselines on various real world datasets. We first perform a semi-synthetic experiment in a realizable setting. Then we use eXtreme Multi-Label Classification (XMC)~\citep{Bhatia16} datasets to test our reduction scheme. The different exploration sampling strategies used in our experiments are~\footnote{Note that all these exploration strategies have been extended to the top-$k$ setting using the ideas in Algorithm~\ref{alg:topk} and many popular contextual bandit algorithms like the ones in~\citep{bietti2018contextual} cannot be easily extended to the top-$k$ setting.}: {\bf Greedy-topk: } The top-$k$ effective arms for each context are chosen greedily according to the regression score; {\bf Boltzmann-topk: }  The top-$(k-r)$ arms are selected greedily. Then the next $r$ arms are selected one by one, each time recomputing the Boltzmann distribution over the remaining arms. Under this sampling scheme the probability of sampling arm $\atil$ is proportional to $\exp(\log (N_{l-1})\beta \ftil(x, \atil))$~\citep{cesa2017boltzmann}; {\bf $\epsilon$-greedy-topk: } Same as above but the last $r$ arms are selected one by one using a scheme where the probability of sampling arm $\atil$ is proportional to $(1 - \epsilon) + \epsilon/A'$ if $\atil$ is the arm with the highest score, otherwise the probability is $\epsilon/A'$ where $A'$ is the number of arms remaining; {\bf \igw-topk: } This is essentially the sampling strategy in Algorithm~\ref{alg:topk}. We set $\gamma_l = \sqrt{CN_{l-1}A'}$ for the $l$-th epoch where $A'$ is the number of remaining arms. 

\begin{figure*}
  \centering
  \subfloat[Realizable eurlex-4k]{\label{fig:rela}\includegraphics[width=0.40\linewidth]{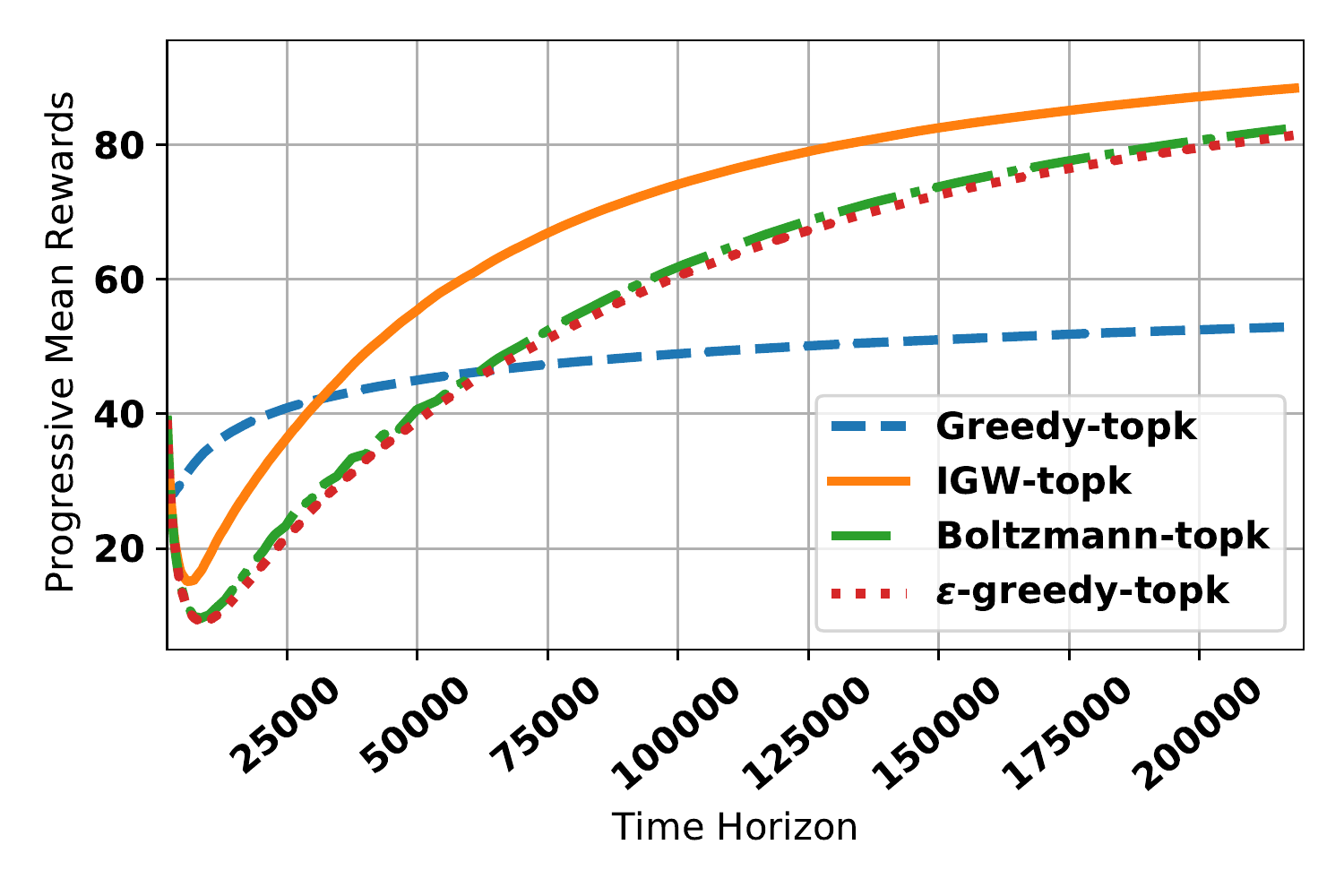}} \hfill
   \subfloat[Inference Time per context on amazon-3m]{\label{fig:inftimes}
   \resizebox{0.40\linewidth}{!}{%
    \begin{tabular}{lc}
    \toprule
    Beam Size (b) &  Inference Time (ms)\\
    \midrule
    10 & 7.85  \\
    30 & 12.84 \\
    100 & 27.83 \\
    2.9K (all arms) & 799.06 \\
    \bottomrule
  \end{tabular}
  }
  } \\
  \subfloat[eurlex-4k]{\label{figur:e4k}\includegraphics[width=0.40\linewidth]{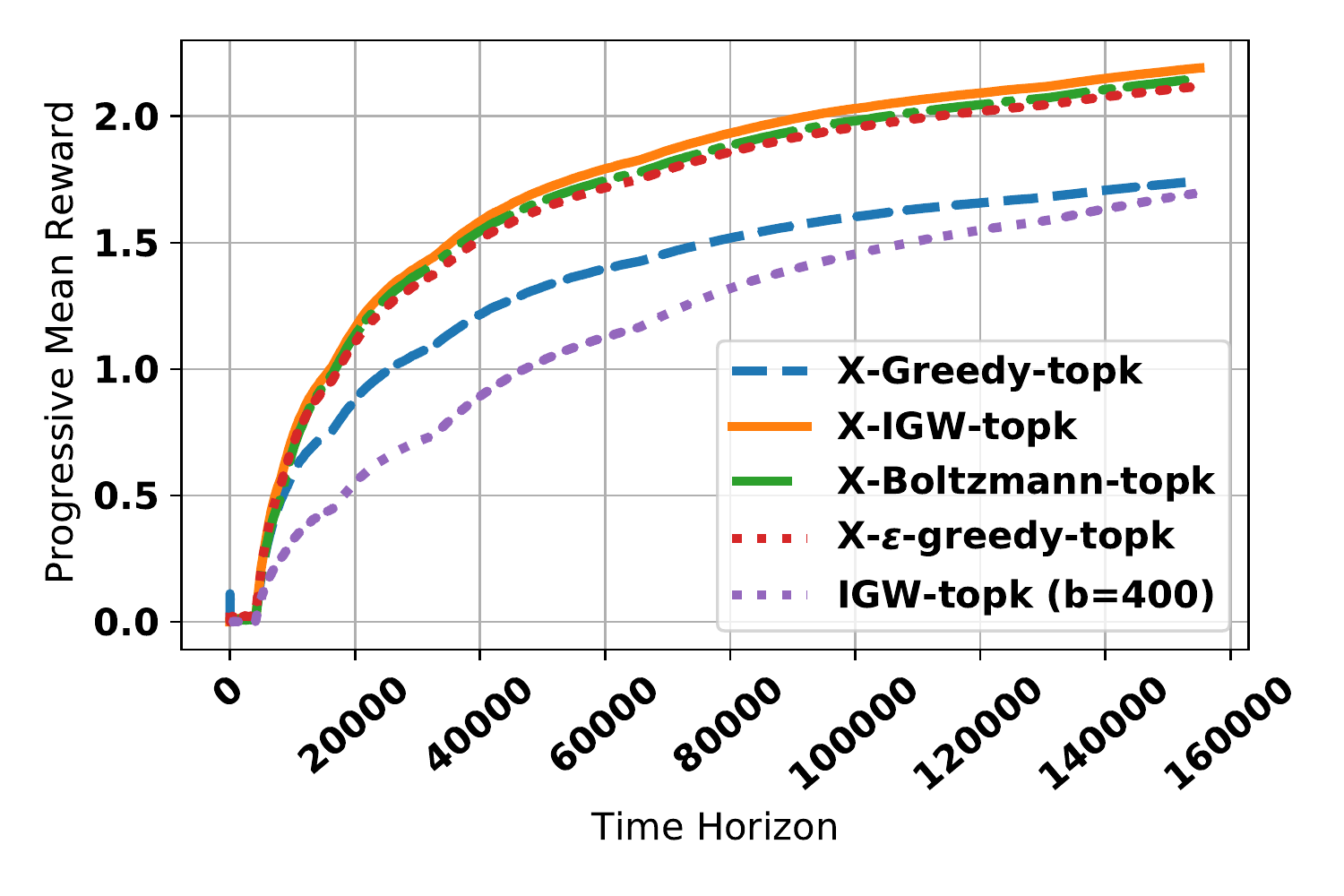}} \hfill
  \subfloat[amazon-3m]{\label{figur:670k}\includegraphics[width=0.40\linewidth]{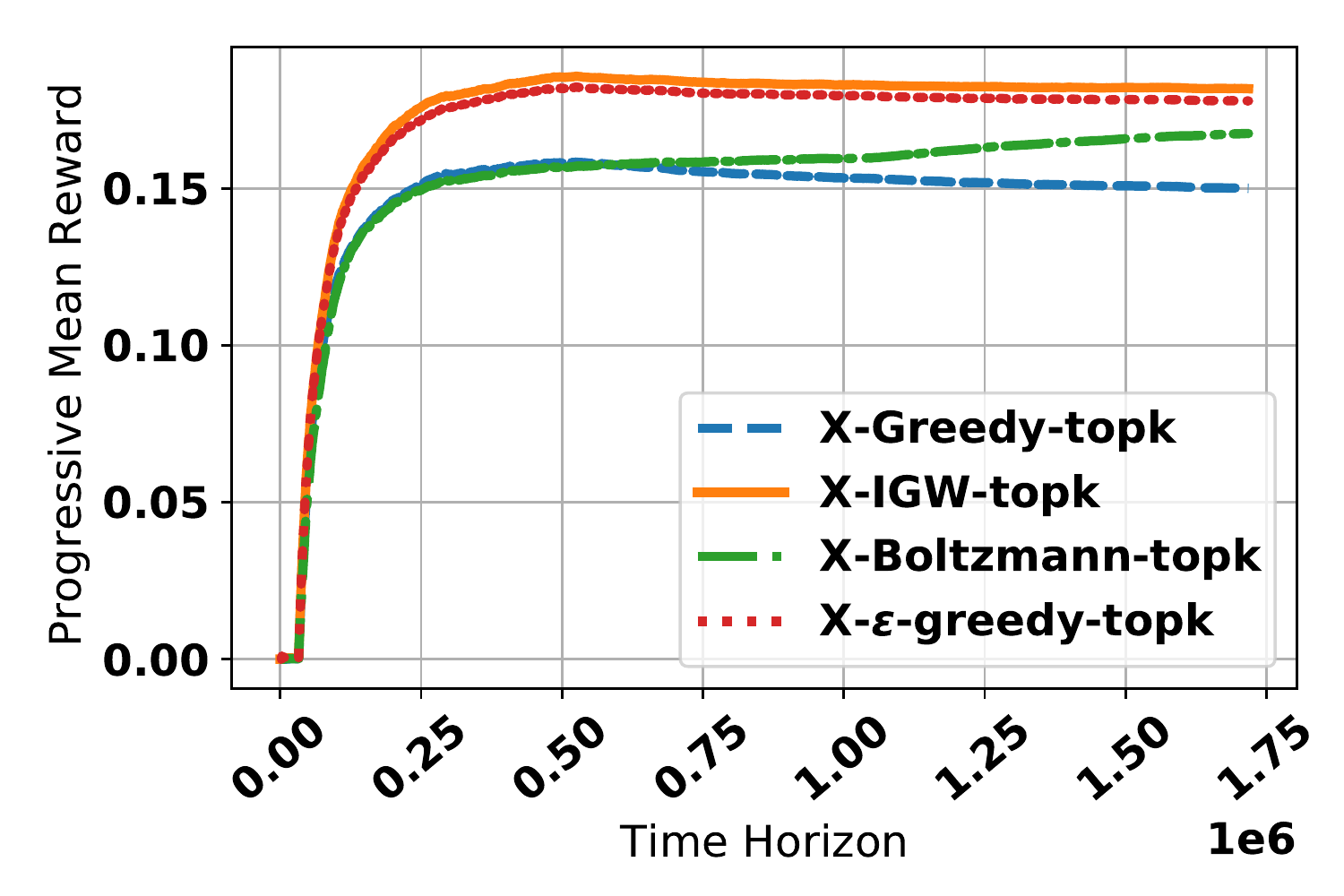}}\\
 \caption{\small In (a) we compare the different sampling strategies on a realizable setting with $k=50$ and $r=25$, derived from the eurlex-4k dataset. In (b) we compare the avg. inference times per context vs different beam sizes on the amazon-3m dataset. Note that for this dataset $b=290,000$ will include all arms in the beam in our setting and is order wise equivalent to no hierarchy. This comparison is done for inference in a setting with $k=5, r=3$. Note that for larger datasets in Table~\ref{tab:stats} our implementation with $b=10, 30$ remains efficient for real-time inference as the time-complexity scales only with the beams-size and the height of the tree. We plot the progressive mean rewards collected by each algorithm as a function of time in two of our 6 datasets in (c)-(d) where the algorithms are implemented under our \xtm reduction framework. In our experiments in (c)-(d) we have $k=5$ and $r=3$. The beam size is 10 except for IGW-topk (b=400) in (c), which serves as a proxy for Algorithm~\ref{alg:topk} without the extreme reduction, as $b=400$ includes all the arms in this dataset.}
 \label{fig:exp}
\end{figure*}

{\bf Realizable Experiment. }
In order to create a realizable setting that is realistic, we choose the eurlex-4k XMC dataset~\citep{Bhatia16} in Table~\ref{tab:stats} and for each arm/label $a\in A$, we fit linear regressor weights $\nu^*_a$ that minimizes $\EE_x[([x; 1.0]^T\nu^*_a$ - $\E[r_a(t) | x])^2]$ over the dataset. Then we consider a derived system where $\E[r_a(t) | x] = [x; 1.0]^T\nu^*_a$ for all $x, a$ that is the learnt weights from before exactly represent the mean rewards of the arms. This system is then realizable for Algorithm~\ref{alg:topk} when the function $\cF$ is linear. Figure~\ref{fig:rela} shows the progressive mean reward (sum of rewards till time $t$ divided by $t$) for all the sampling strategies compared. We see that the \igw sampling strategy in Algorithm~\ref{alg:topk} outperforms all the others by a large margin. For more details please refer to Appendix~\ref{sec:pecos}. Note that the hyper-parameters of all the algorithms are tuned on this dataset in order to demonstrate that even with tuned hyper-parameter choices \igw is the optimal scheme for this realizable experiment. The experiment is done with $k=50, r=25$ and $b=10$.

{\bf \xtm Experiments. }
We now present our empirical results on \xtm multi-label datasets. Our experiments are performed under simulated bandit feedback using real-world \xtm multi-label classification datasets~\citep{Bhatia16}. This experiment startegy is widely used in the literature~\citep{agarwal2014taming, bietti2018contextual} with non-\xtm multi-class datasets (see Appendix~\ref{sec:pecos} for more details). Our implementation uses a hierarchical linear function class inspired by~\citep{pecos2020}. The hyper-parameters in all the algorithms are tuned on the eurlex-4k datasets and then held fixed. This is in line with~\citep{bietti2018contextual}, where the parameters are tuned on a set of datasets and then held fixed.

The tree and routing functions are formed using a small held out section of the datasets, whose sizes are specified in Table~\ref{tab:stats} (Initialization Size). In the interest of space we refer the readers to Appendix~\ref{sec:pecos} for more implementation details.

 We use 6 XMC datasets for our experiments. Table~\ref{tab:stats} provides some basic properties of each dataset. We can see that the number of arms in the largest dataset is as large as 2.8MM. The column Initialization Size denotes the size of the held out set used to intialize our algorithms. Note that for the datasets eurlex-4k and wiki10-31k we bootstrap the original training dataset to a larger size by sampling with replacement, as the original number of samples are too small to show noticeable effects.

\begin{table}[ht!]
    \centering
    \resizebox{0.7\linewidth}{!}{%
\begin{tabular}{lrrrr}
\toprule
      Dataset &  Initialization Size &  Time-Horizon &  No. of Arms & Max. Leaf Size (m) \\
\midrule
     eurlex-4k &                 5000 &        154490 &         4271  & 10\\
 amazoncat-13k &                 5000 &       1186239 &        13330  & 10 \\
    wiki10-31k &                 5000 &        141460 &        30938 & 10 \\
     wiki-500k &                20000 &       1779881 &       501070 & 100 \\
  amazon-670k &                20000 &        490449 &       670091 & 100 \\
     amazon-3m &                50000 &       1717899 &      2812281 & 100 \\
\bottomrule
\end{tabular}
\caption{\small Properties of \xtm Datasets}
\label{tab:stats}
}
\end{table}

\begin{table}[ht!]
    \centering
    \resizebox{0.7\linewidth}{!}{%
\begin{tabular}{lllll}
\toprule
{} &  X-Greedy & X-\igw-topk & X-Boltzmann-topk & X-$\epsilon$-greedy-topk \\
\midrule
X-Greedy         &         - &   0W/0D/6L &         1W/0D/5L &       0W/1D/5L \\
X-\igw-topk       &  6W/0D/0L &          - &         4W/1D/1L &       6W/0D/0L \\
X-Boltzmann-topk &  5W/0D/1L &   1W/1D/4L &                - &       3W/0D/3L \\
X-$\epsilon$-greedy-topk   &  5W/1D/0L &   0W/0D/6L &         3W/0D/3L &              - \\
\bottomrule
\end{tabular}
\caption{ \small Win/Draw/Loss statistics among algorithms for the 6 datasets. When the difference in results between two algorithms is not significant according to the statistical significance formula in~\citep{bietti2018contextual} then it is deemed to be a draw.} 
\label{tab:results}
}
\end{table}

We plot the progressive mean rewards (total rewards collected till time $t$ divided by $t$) for all the algorithms in Figure~\ref{fig:exp} (c)-(d) for two datasets. The rest of the plots are included in Figure~\ref{fig:expapp} in Appendix~\ref{sec:morexp}. The algorithm names are prepended with an $X$ to denote that the sampling is performed under the reduction framework of Algorithm~\ref{alg:hier}. In our experiments the number of arms allowed to be chosen each time is $k=5$. In Algorithm~\ref{alg:hier} we set the number of explore slots $r=3$ and beam-size $b=10$ (unless otherwise specified). We see that all the exploratory algorithms do much better than the greedy version i.e our \xtm reduction framework works for structured exploration when the number of arms are in thousands or millions. The efficacy of the reduction framework is further demonstrated by X-\igw-topk(b=10) being better than \igw-topk (b=400) by 29\% in terms of the mean reward, in Figure~\ref{figur:e4k}. Note that here \igw-topk(b=400) serves as a proxy for Algorithm~\ref{alg:topk} directly applied without the hierarchy, as the beam includes all the arms. The \igw scheme is always among the top $2$ strategies in all datasets. It is the only strategy among the baselines that has optimal theoretical performance and this shows that the algorithm is practical. Table~\ref{tab:results} provides Win(W)/Draw(D)/Loss(L) for each algorithm against the others. We use the same W/D/L scheme as in~\citep{bietti2018contextual} to create this table. Note that X-\igw-topk has the highest win percentage overall. In Figure~\ref{fig:inftimes} we compare the inference times for \igw of our hierarchical linear implementation for different beam-sizes on amazon-3m. Note that $b=2.9 K$ will include all arms in this dataset and is similar to a flat hierarchy. This shows that our algorithm will remain practical for real time inference on large datasets when $b \leq 30$ is used.

 \section{Discussion}
 We provide regret guarantees for the top-$k$ arm selection problem in realizable contextual bandits with general function classes. The algorithm can be theoretically and practically extended to extreme number of arms under our proposed reduction framework which models a practically motivated arm hierarchy. We benchmark our algorithms on XMC datasets under simulated bandit feedback. There are interesting directions for future work, for instance extending the analysis to a setting where the reward derived from the $k$ arms is a set function with interesting structures such as sub-modularity. It would also be interesting to analyze the \xtm setting where the routing functions and hierarchy can be updated in a data driven manner after every few epochs.
\FloatBarrier

\bibliography{xcb}
\clearpage
\appendix

\section{Top-$k$ Analysis}
\label{sec:topkanalysis}
\paragraph{Notation:}
Let $l$ denote epoch index with $n_l$ time steps. Define $N_{l}=\sum_{i=1}^l n_i$. At the beginning of each epoch $l$, we compute $\pred_l(x,a)$ as regression with respect to past data,
$$\pred_l = \argmin_{f\in\cF} \sum_{t=1}^{N_{l-1}} \sum_{a\in \feedbackset_t} (f(x_{t}, a)-r_{t}(a))^2,$$
where $\feedbackset_t$ is the subset for which the learner receives feedback. 

Let $\{\phi_l\}_{l\geq 2}$ be a sequence of numbers. The analysis in this section will be carried out under the event
\begin{align}
    \label{eq:big_event}
    \cE = \left\{ l \geq 2: \frac{2}{N_{l-1}}\sum_{s = 1}^{N_{l-1}} \EE_{x_s,\chosenset_s}\left\{ \frac{1}{k}\sum_{a\in\chosenset_s}  (\pred_l(x_s,a) - f^*(x_s,a))^2  \vert \cH_{s-1} \right\}
    \leq \phi_l^2 \right\}
\end{align}
Lemmas~\ref{lem:reg} and \ref{lem:reg2} compute $\phi_l$ for finite class $\cF$, such that event $\cE$ holds with high probability. 

We define $\gamma_l = \sqrt{A-k+1}/(32\phi_l)$, the scaling parameter used by Algorithm~\ref{alg:topk}. In this paper, we analyze Algorithm~\ref{alg:topk} with $r=1$, i.e. our procedure deterministically selects top $k-1$ actions of $\pred_l$ and selects the remaining action according to Inverse Gap Weighting on the remaining coordinates. 

A deterministic strategy $\alpha$ is a map $\alpha:\cX\to\cA$. Throughout the proofs, we employ the following shorthand to simplify the presentation. We shall write $\pred_i(x,\alpha)$ and $f^*(x,\alpha)$ in place of $\pred_i(x,\alpha(x))$ and $f^*(x,\alpha(x))$. We reserve the letter $\alpha$ for a strategy and $a$ for an action.

Given $x$, we let $\widehat{\alpha}_l^j(x)$ be the $j$-th highest action according to $\pred_l(x,\cdot)$. Similarly, $\alpha^{*,j}(x)$ is the $j$-th highest action according to $f^*(x,\cdot)$. We say that the set of strategies $\alpha^1,\ldots,\alpha^k$ is non-overlapping if for any $x$ the set $\{\alpha^1(x),\ldots,\alpha^k(x)\}$ is a set of distinct actions. Let $e(s)$ denote the epoch corresponding to time step $s$. 

Our argument is based on the beautiful observation of \citep{simchi2020bypassing} that one can analyze \igw inductively, by controlling the differences between estimated gaps (to the best estimated action) and the true gaps (to the best true action in the given context), with a \textit{mismatched factor of $2$}. We extend this technique to top-$k$ selection, which introduces a number of additional difficulties in the analysis.

\paragraph{Induction hypothesis ($l$):} For any epoch $i<l$, and all non-overlapping strategies $\alpha^1,\ldots,\alpha^k \in \cA^\cX$,
$$\E_x \left\{ \sum_{j=1}^k [\pred_i(x,\widehat{\alpha}_i^j) -\pred_i(x,\alpha^j) ] - 2 \sum_{j=1}^k [f^*(x,\alpha^{*,j})-f^*(x,\alpha^j)] \right\} \leq \frac{k(A -k+1)}{\gamma_i}$$
and
$$\E_x  \left\{ \sum_{j=1}^k [f^*(x,\alpha^{*,j})-f^*(x,\alpha^j)] - 2\sum_{j=1}^k[\pred_i(x,\widehat{\alpha}_i^j) -\pred_i(x,\alpha^j) ] \right\} \leq \frac{k(A -k+1)}{\gamma_i}.$$

\begin{lemma}
    \label{lem:main}
    Suppose event \eqref{eq:big_event} holds. For all non-overlapping strategies $\alpha^1,\ldots,\alpha^k$,
    \begin{align*}
    \E_{x} \frac{1}{k}\sum_{j=1}^k  |\pred_l(x,\alpha^j)-f^*(x,\alpha^j)|  &\leq  \phi_l  \cdot \left((A -k+1) +\sum_{i=1}^{l-1} \frac{n_i}{N_{l-1}}\gamma_i  \E_{x} \frac{1}{k}\sum_{j=1}^k \left[ \pred_i(x,\widehat{\alpha}^j_i) - \pred_i(x,\alpha^j)\right]  \right)^{1/2}
    \end{align*}
Hence, by the induction hypothesis $(l)$,
\begin{align*}
    \E_{x} \frac{1}{k}\sum_{j=1}^k  |\pred_l(x,\alpha^j)-f^*(x,\alpha^j)|  &\leq \sqrt{2}\phi_l  \cdot \left((A -k+1)+ \gamma_l  \E_{x} \frac{1}{k}\sum_{j=1}^k \left[ f^*(x,\alpha^{*,j}) - f^*(x,\alpha^j)\right]  \right)^{1/2} 
\end{align*}
assuming $\gamma_i$ are non-decreasing.
\end{lemma}
\begin{proof}
Given $x$, let $T_x(\pred_{i})\subset [A]$ denote the indices of top $k-1$ actions according to $\pred_i(x,\cdot)$. Let $p_i(\cdot| x)$ denote the \igw distribution on epoch $i$, with support on the remaining $A-k+1$ actions. On round $s$ in epoch $e(s)$, given $x_s$, Algorithm~\ref{alg:topk} with $r=1$ chooses $\chosenset_s$ by selecting $T_{x_s}(\pred_{e(s)})$ determistically and selecting the last action according to $p_{e(s)}(\cdot| x_s)$. We write $p_{e(s)}(\alpha|x_s)$ as a shorthand for $p_{e(s)}(\alpha(x_s)|x_s)$.

For non-overlapping strategies $\alpha^1,\ldots,\alpha^k$,
\begin{align*}
    &\E_{x} \frac{1}{k}\sum_{j=1}^k |\pred_l(x,\alpha^j)-f^*(x,\alpha^j)| = \frac{1}{N_{l-1}}\sum_{s=1}^{N_{l-1}} \E_{x_s} \left\{ \frac{1}{k}\sum_{j=1}^k\left|  \pred_{l}(x_s,\alpha^j)-f^*(x_s,\alpha^j)\right| \right\}.
\end{align*}
This sum can be written as
\begin{align*}
    &\frac{1}{N_{l-1}}\sum_{s=1}^{N_{l-1}} \E_{x_s} \left[ \frac{1}{k}\sum_{j=1}^k\left|  \pred_{l}(x_s,\alpha^j)-f^*(x_s,\alpha^j)\right|\cdot \ind{\alpha^j(x_s)\in T_{x_s}(\pred_{e(s)})} \right.\\
    &\left.\hspace{0in}+   \frac{1}{k}\sum_{j=1}^k |\pred_{l}(x_s,\alpha^j)-f^*(x_s,\alpha^j)| \sqrt{p_{e(s)}(\alpha^j|x_s)} \frac{1}{\sqrt{p_{e(s)}(\alpha^j|x_s)}} \cdot \ind{\alpha^j(x_s)\notin T_{x_s}(\pred_{e(s)})} \right].
\end{align*}
By the Cauchy-Schwartz inequality, the last expression is upper-bounded by
\begin{align*}
    &\left(\frac{1}{N_{l-1}}\sum_{s=1}^{N_{l-1}} \E_{x_s} \frac{1}{k}\sum_{j=1}^k |f^*(x_s,\alpha^j)-\pred_{l}(x_s,\alpha^j)|^2 \ind{\alpha^j(x_s)\in T_{x_s}(\pred_{e(s)})} \right)^{1/2}\\
    &+\left(\frac{1}{N_{l-1}}\sum_{s=1}^{N_{l-1}} \E_{x_s} \frac{1}{k}\sum_{j=1}^k |f^*(x_s,\alpha^j)-\pred_{l}(x_s,\alpha^j)|^2 p_{e(s)}(\alpha^j|x_s) \ind{\alpha^j(x_s)\notin T_{x_s}(\pred_{e(s)})} \right)^{1/2} \\
    &\hspace{2in}\times \left(\frac{1}{N_{l-1}}\sum_{s=1}^{N_{l-1}} \E_{x_s} \frac{1}{k}\sum_{j=1}^k \frac{1}{p_{e(s)}(\alpha^j|x_s)} \ind{\alpha^j(x_s)\notin T_{x_s}(\pred_{e(s)})} \right)^{1/2} \\
    &\leq \left(\frac{1}{N_{l-1}}\sum_{s=1}^{N_{l-1}} \E_{x_s}\frac{1}{k}\sum_{a\in T_{x_s}(\pred_{e(s)})} |f^*(x_s,a)-\pred_{l}(x_s,a)|^2  \right)^{1/2}\\
    &+\left(\frac{1}{N_{l-1}}\sum_{s=1}^{N_{l-1}} \frac{1}{k} \E_{x_s,a\sim p_{e(s)}(\cdot|x_s)} |f^*(x_s,a)-\pred_{l}(x_s,a)|^2   \right)^{1/2} \\
    &\hspace{2in}\times\left(\sum_{i=1}^{l-1} \frac{n_i}{N_{l-1}}\E_{x}\frac{1}{k}\sum_{j=1}^k  \frac{1}{p_{i}(\alpha^j|x)} \ind{\alpha^j(x)\notin T_{x}(\pred_{i})} \right)^{1/2}. 
\end{align*}
We further upper bound the above by
\begin{align}
    &\left\{ \left(\frac{1}{N_{l-1}}\sum_{s=1}^{N_{l-1}} \E_{x_s}\frac{1}{k}\sum_{a\in T_{x_s}(\pred_{e(s)})} |f^*(x_s,a)-\pred_{l}(x_s,a)|^2  \right)^{1/2} \right. \notag\\
    &\left.\hspace{1in}+\left(\frac{1}{N_{l-1}}\sum_{s=1}^{N_{l-1}} \frac{1}{k} \E_{x_s,a\sim p_{e(s)}(\cdot|x_s)} |f^*(x_s,a)-\pred_{k}(x_s,a)|^2   \right)^{1/2} \right\} \notag\\
    &\hspace{2.2in}\times\left(1 \vee \sum_{i=1}^{l-1} \frac{n_i}{N_{l-1}}\E_{x} \frac{1}{k} \sum_{j=1}^k \frac{1}{p_{i}(\alpha^j|x)} \ind{\alpha^j\notin T_x(\pred_{i})} \right)^{1/2} \notag\\
    &\leq \left(\frac{2}{N_{l-1}}\sum_{s=1}^{N_{l-1}} \frac{1}{k}\E_{x_s} \left[ 
        \sum_{a\in T_{x_s}(\pred_{e(s)})} |f^*(x_s,a)-\pred_{l}(x_s,a)|^2  + \E_{a\sim p_{e(s)}(\cdot|x_s)} |f^*(x_s,a)-\pred_{l}(x_s,a)|^2
        \right]
        \right)^{1/2} \label{eq:stat_rate_in_proof}\\
    &\hspace{2.2in}\times\left(1\vee \sum_{i=1}^{l-1} \frac{n_i}{N_{l-1}}\E_{x} \frac{1}{k}\sum_{j=1}^k \frac{1}{p_{i}(\alpha^j|x)} \ind{a^j(x)\notin T_x(\pred_{i})} \right)^{1/2} \label{eq:inverse_term_in_proof}
\end{align}
where we use $(\sqrt{a}+\sqrt{b})^2\leq 2(a+b)$ for nonnegative $a,b$. Now, observe that 
\begin{align}
    &\E_{x_s} \left[ 
         \sum_{a\in T_{x_s}(\pred_{e(s)})} |f^*(x_s,a)-\pred_{l}(x_s,a)|^2  + \E_{a\sim p_{e(s)}(\cdot|x_s)} |f^*(x_s,a)-\pred_{l}(x_s,a)|^2
         \right] \\
    &=\E_{x_s,\chosenset_s}\left\{ \sum_{a\in\chosenset_s}  (\pred_l(x_s,a) - f^*(x_s,a))^2  \vert \cH_{s-1} \right\}
\end{align}
by the definition of the selected set $\chosenset_s$ in Algorithm~\ref{alg:topk} with $r=1$. Under the event \eqref{eq:big_event}, the expression in \eqref{eq:stat_rate_in_proof} is at most $\phi_l$.
We now turn to the expression in \eqref{eq:inverse_term_in_proof}. Note that by definition, for any strategy $\alpha^j$
\begin{align*}
    \frac{1}{p_i (\alpha^j|x)} \ind{a^j(x)\notin T_{x}(\pred_i)} &= \left[ (A -k+1) + \gamma_i (\pred_i(x,\widehat{\alpha}^k_i) - \pred_i(x,\alpha^j))\right]\ind{\alpha^j(x)\notin T_{x}(\pred_i)} \\
    &\leq (A -k+1) + \gamma_i\left[ \pred_i(x,\widehat{\alpha}^k_i) - \pred_i(x,\alpha^j)\right]_+ \, ,
\end{align*}
where $[a]_+ = \max\{a,0\}$. Therefore, by Lemma~\ref{lem:induct}, for any non-overlapping strategies $\alpha^1,\ldots,\alpha^{k}$,
\begin{align*}
    \frac{1}{k} \sum_{j=1}^k \frac{1}{p_i (\alpha^j|x)} \ind{\alpha^j(x)\notin T_x(\pred_i)} &\leq  (A -k+1) + \frac{1}{k} \sum_{j=1}^k \gamma_i\left[ \pred_i(x,\widehat{\alpha}^k_i) - \pred_i(x,\alpha^j)\right]_+ \\
    &\leq (A -k+1) + \frac{1}{k} \sum_{j=1}^k \gamma_i\left[ \pred_i(x,\widehat{\alpha}^j_i) - \pred_i(x,\alpha^j)\right].
\end{align*}
Since the above expression is at least $(A -k+1)\geq 1$, we may drop the maximum with 1 in \eqref{eq:inverse_term_in_proof}. Putting everything together,
\begin{align*}
    \E_{x} \frac{1}{k}\sum_{j=1}^k  |\pred_l(x,\alpha^j)-f^*(x,\alpha^j)|  &\leq \phi_l  \cdot \left((A -k+1) +\sum_{i=1}^{l-1} \frac{n_i}{N_{l-1}}\gamma_i  \E_{x} \frac{1}{k}\sum_{j=1}^k \left[ \pred_i(x,\widehat{\alpha}^j_i) - \pred_i(x,\alpha^j)\right]  \right)^{1/2}
\end{align*}
To prove the second statement, by induction we upper bound the above expression by
\begin{align*}
    &\phi_l  \cdot \left((A -k+1) +\max_{i<l} \gamma_i  \left\{ 2\E_{x} \frac{1}{k}\sum_{j=1}^k \left[ f^*(x,\alpha^{*,j}) - f^*(x,\alpha^j)\right] + \frac{A}{\gamma_i} \right\} \right)^{1/2} \\
    &\leq \phi_l  \cdot \left(2(A -k+1) + 2\gamma_l  \E_{x} \frac{1}{k}\sum_{j=1}^k \left[ f^*(x,\alpha^{*,j}) - f^*(x,\alpha^j)\right]  \right)^{1/2}. 
\end{align*}

\end{proof}

We now prove that inductive hypothesis holds for each epoch $l$.
\begin{lemma}
    Suppose we set $\gamma_l = \sqrt{A-k+1}/(32\phi_l)$ for each $l$, and that event $\cE$ in \eqref{eq:big_event} holds. Then the induction hypothesis holds for each $l\geq 2$.
\end{lemma}
\begin{proof}

The base of the induction ($l=2$) is satisfied trivially if $\gamma_2=O(1)$ since functions are bounded. Now suppose the induction hypothesis $(l)$ holds for some $l\geq 2$. We shall prove it for $(l+1)$.

Denote by $\balpha=(\alpha^1,\ldots,\alpha^{k})$ any set of non-overlapping strategies. We also use the shorthand $A' = A -k+1$ for the size of the support of the IGW distribution. Define
$$\RegExp(\balpha) = \E_{x} \frac{1}{k}\sum_{j=1}^k \left[ f^*(x, \alpha^{*,j}) - f^*(x,\alpha^j) \right],~~~~ \RegEmp_l(\balpha) = \E_{x} \frac{1}{k}\sum_{j=1}^k \left[ \pred_l(x, \widehat{\alpha}^j_l) - \pred_k (x,\alpha^j) \right].$$
Since
$$\left[ f^*(x, \alpha^{*,j}) - f^*(x,a) \right] = \left[ \pred_l(x, \alpha^{*,j}) - \pred_l (x,a) \right] + \left[ f^*(x, \alpha^{*,j}) - \pred_l (x,\alpha^{*,j}) \right] + \left[ \pred_l(x, a) - f^* (x,a) \right],$$
it holds that 
\begin{align*}
    &\sum_{j=1}^k \left[ f^*(x, \alpha^{*,j}) - f^*(x,\alpha^j) \right] \\
    &= \sum_{j=1}^k \left[ \pred_l(x, \alpha^{*,j}) - \pred_l (x,\alpha^j) \right] + \sum_{j=1}^k \left[ f^*(x, \alpha^{*,j}) - \pred_l (x,\alpha^{*,j}) \right] + \sum_{j=1}^k \left[ \pred_l(x, \alpha^j) - f^* (x,\alpha^j) \right] \\ 
    &\leq \sum_{j=1}^k \left[ \pred_l(x, \widehat{\alpha}^j_l) - \pred_l (x,\alpha^j) \right] + \sum_{j=1}^k \left[ f^*(x, \alpha^{*,j}) - \pred_l (x,\alpha^{*,j}) \right] + \sum_{j=1}^k \left[ \pred_l(x, \alpha^j) - f^* (x,\alpha^j) \right].
\end{align*}
Therefore, for any $\balpha$,
\begin{align}
    \label{eq:RlessRemp}
    \RegExp(\balpha)
    &\leq \E_{x}  \frac{1}{k}\sum_{j=1}^k \left[ \pred_l(x, \alpha^{*,j}) - \pred_l (x,\alpha^j) \right] \notag \\
    &\hspace{1in} +  \E_x \frac{1}{k}\sum_{j=1}^k  \left[ f^*(x, \alpha^{*,j}) - \pred_l (x,\alpha^{*,j}) \right] + \E_{x} \frac{1}{k}\sum_{j=1}^k  \left[ \pred_l(x, \alpha^j) - f^* (x,\alpha^j) \right].
\end{align}
For the middle term in \eqref{eq:RlessRemp}, we apply the last statement of Lemma~\ref{lem:main} to $\alpha^{*,1},\ldots,\alpha^{*,k}$. We have:
\begin{align*}
    \E_x \frac{1}{k}\sum_{j=1}^k  \left[f^*(x, \alpha^{*,j}) - \pred_l (x,\alpha^{*,j}) \right] \leq \sqrt{2A'}\phi_l\cdot
\end{align*}
For the last term in \eqref{eq:RlessRemp},
\begin{align*}
     \frac{1}{k}\sum_{j=1}^k \E_{x}  \left[ \pred_l(x, \alpha^j) - f^* (x,\alpha^j) \right] \leq \sqrt{2} \phi_l \cdot \left(A' + \gamma_l \RegExp(\balpha) \right)^{1/2}.
\end{align*}
Hence, we have the inequality
\begin{align*}
    \RegExp(\balpha) &\leq \RegEmp_l(\balpha) + \sqrt{2A'}\phi_l + \sqrt{2}\phi_l \cdot \left(A' +\gamma_l \RegExp(\balpha) \right)^{1/2} \\
    &\leq \RegEmp_l(\balpha) + 2\phi_l \sqrt{2A'} +  \phi_l\sqrt{2\gamma_l \RegExp(\balpha) }\\
    &\leq \RegEmp_l(\balpha) + 2\phi_l \sqrt{2A'} + \gamma_l \phi_l^2 + \frac{1}{2}\RegExp(\balpha) 
\end{align*}
and thus
\begin{align*}
    \RegExp(\balpha) &\leq 2\RegEmp_l(\balpha) + 4\phi_l\sqrt{2A'} +  2\gamma_l  \phi_l^2 \leq 2\RegEmp_l(q) + A'/(2\gamma_l)
\end{align*}
On the other hand,
\begin{align*}
    \left[ \pred_l(x, \widehat{\alpha}^j_l) - \pred_l (x,\alpha^j) \right] = \left[ f^*(x, \widehat{\alpha}^j_l) - f^*(x,\alpha^j) \right] + \left[ \pred_l(x, \widehat{\alpha}^j_l) - f^*(x, \widehat{\alpha}^j_l) \right] + \left[  f^*(x,\alpha^j) - \pred_l (x,\alpha^j)  \right] 
\end{align*}
and so
\begin{align*}
    &\sum_{j=1}^k\left[ \pred_l(x, \widehat{\alpha}^j_l) - \pred_l (x,\alpha^j) \right] \\
    &= \sum_{j=1}^k\left[ f^*(x, \widehat{\alpha}^j_l) - f^*(x,\alpha^j) \right] + \sum_{j=1}^k\left[ \pred_l(x, \widehat{\alpha}^j_l) - f^*(x, \widehat{\alpha}^j_l) \right] + \sum_{j=1}^k \left[  f^*(x,\alpha^j) - \pred_l (x,\alpha^j)  \right] \\
    &\leq \sum_{j=1}^k\left[ f^*(x, \alpha^{*,j}) - f^*(x,\alpha^j) \right] + \sum_{j=1}^k\left[ \pred_l(x, \widehat{\alpha}^j_l) - f^*(x, \widehat{\alpha}^j_l) \right] + \sum_{j=1}^k\left[  f^*(x,\alpha^j) - \pred_l (x,\alpha^j)  \right].
\end{align*}
Therefore, for any $\balpha$
\begin{align}
    \label{eq:RemplessR}
    \RegEmp_l(\balpha) &\leq \RegExp(\balpha) + \E_{x} \frac{1}{k}\sum_{j=1}^k [\pred_l(x,\widehat{\alpha}_l^j) - f^*(x,\widehat{\alpha}_l^j)] +  \E_{x} \frac{1}{k}\sum_{j=1}^k [ f^*(x,\alpha^j)-\pred_l(x,\alpha^j)].
\end{align}
The last term in \eqref{eq:RemplessR} is bounded by Lemma~\ref{lem:main} by
\begin{align*}
    \E_{x} \frac{1}{k}\sum_{j=1}^k | f^*(x,\alpha^j)-\pred_l(x,\alpha^j)| &\leq \sqrt{2} \phi_l \cdot \left(A' + \gamma_l \RegExp(\balpha) \right)^{1/2} \\
    &\leq  \sqrt{2} \phi_l \cdot \left(A' + 2\gamma_l \RegEmp_l(\balpha) + A'/2 \right)^{1/2} \\
    &\leq 2\phi_l \sqrt{A'} + 2\phi_l^2 \gamma_l + \frac{1}{2}\RegEmp_l(\balpha) \\
    &\leq \frac{A'}{4\gamma_l } + \frac{1}{2}\RegEmp_l(\balpha).
\end{align*}
Now, for the middle term in \eqref{eq:RemplessR}, we use the above inequality with $\widehat{\balpha}_l = (\widehat{\alpha}^{1}_l, \ldots, \widehat{\alpha}^{k}_l)$:
\begin{align*}
    \E_{x} \frac{1}{k}\sum_{j=1}^k [\pred_l(x,\widehat{\alpha}_l^j) - f^*(x,\widehat{\alpha}_l^j)] \leq \frac{A'}{4\gamma_l} + \frac{1}{2} \RegEmp_l(\widehat{\balpha}_l) = \frac{A'}{4\gamma_l}.
\end{align*}
Putting the terms together,
\begin{align*}
    \RegEmp_l(\balpha) &\leq 2\RegExp(\balpha) + \frac{A'}{\gamma_l }.
\end{align*}    
Since $\balpha$ is arbitrary, the induction step follows.
\end{proof}

\begin{lemma}
\label{lem:induct}
    For $v\in\reals^A$, let $\widehat{a}^1,\ldots,\widehat{a}^k$ be indices of largest $k$ coordinates of $v$ in decreasing order. Let $a^1,\ldots,a^k$ be any other set of distinct coordinates. Then
    $$\sum_{j=1}^k [v(\widehat{a}^k)-v(a^j)]_+ \leq \sum_{j=1}^k v(\widehat{a}^j)-v(a^j) $$
\end{lemma}
\begin{proof}
    We prove this by induction on $r$. For $r=1$, 
    $$[v(\widehat{a}^1)-v(a^1)]_+ = v(\widehat{a}^1)-v(a^1) $$
    Induction step: Suppose 
    $$\sum_{j=1}^{k-1} [v(\widehat{a}^k)-v(b^j)]_+ \leq \sum_{j=1}^{k-1} v(\widehat{a}^j)-v(b^j) $$
    for any $b^{1},\ldots,b^{k-1}$. Let $a^{m}=\argmin_{j=1,\ldots,k} v(a^j)$. Since all the values are distinct, it must be that $v(\widehat{a}^k)\geq v(a^m)$. Applying the induction hypothesis to $\{a^1,\ldots,a^k\}\setminus \{a^m\}$ and adding 
    $$[v(\widehat{a}^k) - v(a^m)]_+ = v(\widehat{a}^k) - v(a^m)$$
    to both sides concludes the induction step.
\end{proof}

\begin{proof}[Proof of Theorem~\ref{thm:topkregret}]
Recall that on epoch $l$, the strategy is $\alpha^1_l=\widehat{\alpha}^1_l,\ldots,\alpha^{k-1}_l=\widehat{\alpha}^{k-1}_l$ for the first $k-1$ arms, and then sampling $\alpha^k_l(x)$ from \igw distribution $p_l$. Observe that for any $x$ and any draw $\alpha^k_l(x)$, the set of $k$ arms is distinct (i.e. the strategies are non-overlapping), and thus under the event $\cE$ in \eqref{eq:big_event}, Lemma~\ref{lem:main} and inductive statements hold. Hence, expected regret per step in epoch $l$ is bounded as
\begin{align}
     & \E_{x,\alpha_l^k(x)} \sum_{j=1}^k [f^*(x,\alpha^{*,j}) - f^*(x,\alpha^j_l)] \\
     &\leq \frac{k(A-k+1)}{\gamma_l} + 2\E_{x,\alpha_l^k(x)} \sum_{j=1}^k [\pred_l(x,\widehat{\alpha}^j_l)) - \pred_l(x,\alpha^j)]  \nonumber \\
    &= \frac{k(A-k+1)}{\gamma_l}+ 2\E_{x,\alpha_l^k(x)} [\pred_l(x,\widehat{\alpha}^k_l)) - \pred_l(x,\alpha^k_l)]  \nonumber \\
    &\leq \frac{k(A -k+1)}{\gamma_l} + 2\E_x\sum_{a\notin T_x(\pred_l)} \frac{\pred_l(x,\widehat{\alpha}^k_l) - \pred_l(x,a)}{(A -k+1) +\gamma_l [\pred_l(x,\widehat{\alpha}_l^k) - \pred_l(x,a)]} \nonumber \\
    &\leq   \frac{k(A -k+1)}{\gamma_l} + \frac{2(A-k+1)}{\gamma_l} \label{eq:epoch} 
\end{align}

From Lemma~\ref{lem:reg}, the event $\cE$ in \eqref{eq:big_event} holds with probability at least $1-\delta$ if we set
\begin{align*}
    \phi_l = \sqrt{\frac{162}{cN_{l-1}} \log \left( \frac{|\cF| N_{l-1}^3}{\delta}\right)} .
\end{align*}
Now recall that we set $N_{l} = 2^l \leq 2T$ and  $\gamma_l = \sqrt{A-k+1}/(32\phi_l)$. Combining this with equation~\eqref{eq:epoch}, we find that the cumulative regret is bounded with probability at least $1-\delta$ by
\begin{align*}
    R(T) &\leq \sum_{l=2}^{e(T)} \frac{(k+2)(A -k + 1)N_{l-1}}{\gamma_l} \\
    &\leq c^{-1/2} 408(k+2)\sqrt{(A - k + 1) \log \left( \frac{|\cF| T^3}{\delta}\right)} \sum_{l=2}^{\log_2(2T)} 2^{(l-1)/2} \\
    &\leq c^{-1/2}  2308(k+2) \sqrt{(A - k + 1) T \log \left( \frac{|\cF| T^3}{\delta}\right)}.
\end{align*}
\end{proof}

\begin{proof}[Proof of Theorem~\ref{thm:topkregret2}]
The proof is essentially the same as the proof of Theorem~\ref{thm:topkregret}.

From Lemma~\ref{lem:reg2}, the event $\cE$ in \eqref{eq:big_event} holds with probability at least $1-\delta$ if we set
\begin{align*}
    \phi_l = \sqrt{\frac{420}{cN_{l-1}} \log \left( \frac{|\cF| N_{l-1}^3}{\delta}\right) + 2\epsilon^2} .
\end{align*}
Combining this with equation~\eqref{eq:epoch} we get that the regret is bounded by,
\begin{align*}
    &R(T) \leq \sum_{l=2}^{e(T)} \frac{(k+2)(A -k + 1)N_{l-1}}{\gamma_l} \\
    &\leq c^{-1/2} 656(k+2)\sqrt{(A - k + 1) \log \left( \frac{|\cF| T^3}{\delta}\right)} \sum_{l=2}^{\log_2(2T)} 2^{(l-1)/2} + 46(k+2)\sqrt{(A - k + 1)\epsilon^2} \sum_{l=2}^{e(T)} N_{l-1}\\
    &\leq c^{-1/2}3711(k+2) \sqrt{(A - k + 1) T \log \left( \frac{|\cF| T^3}{\delta}\right)} + 46(k+2)T\sqrt{(A - k + 1)\epsilon^2}
\end{align*}
given $\cE$ is true.
\end{proof}

\section{Regression Martingale Bound}

Recall that we have the following dependence structure in our problem. On each round $s$, context $x_s$ is drawn independently of the past $\cH_{s-1}$ and rewards $\br_s=\{r_s(a)\}_{a\in\cA}$ are drawn from the distribution with mean $f^*(x_s,a)$. The algorithm selects a random set $\chosenset_s$ given $x_s$, and feedback is provided for a (possibly random) subset $\feedbackset_s\subseteq \chosenset$. Importantly, $\chosenset_s$ and $\feedbackset_s$ are independent of $\br_s$ given $x_s$.

The next lemma considers a single time step $s$, conditionally on the past $\cH_{s-1}$.
\begin{lemma}
\label{lem:variance}
    Let $x_s, \br_s=\{r_s(a)\}_{a\in\cA}$ be sampled from the data distribution, and let $\chosenset_s\subseteq\cA$ be conditionally independent of $\br_s$ given $x_s$. Let $\feedbackset_s\subseteq\chosenset_s$ be a random subset given $\chosenset_s$ and $x_s$, but independent of $\br_s$. Fix an arbitrary $f:\cX\times\cA\to[0,1]$ and define the following random variable,
\begin{align*}
    Y_s = \frac{1}{k} \sum_{a\in \cA} \left((f(x_s,a) - r_s(a))^2 - (f^*(x_s,a) - r_s(a))^2\right)\times\ind{a\in\Phi_s}.
\end{align*}
Then, under the realizability assumption (Assumption~\ref{asum:realizability}), we have the following,
\begin{align*}
    \EE_{x_s, \br_s, \chosenset_s, \feedbackset_s} [Y_s] &= \frac{1}{k} \sum_{a\in\cA} \EE_{x_s,\chosenset_s,\feedbackset_s}\left\{ (f(x_s,a) - f^*(x_s,a))^2\times \ind{a\in\feedbackset_s} \right\} 
\end{align*}
and
\begin{align*}
    \VV_{x_s,\br_s,\chosenset_s,\feedbackset_s} [Y_s] &\leq 4 \EE_{x_s,\br_s,\chosenset_s,\feedbackset_s} [Y_s].
\end{align*}
\end{lemma}
\begin{proof}
By the conditional independence assumptions, 
\begin{align*}
    \EE_{x_s,\br_s,\chosenset_s,\feedbackset_s} [Y_s] &= \frac{1}{k}\sum_{a\in\cA}\EE_{x_s,\br_s,\chosenset_s,\feedbackset_s} \left\{ (f(x_s,a) - f^*(x_s,a))(f(x_s,a) + f^*(x_s,a) -2 r_s(a)) \times \ind{a\in\feedbackset_s} \right\} \\
    &= \frac{1}{k} \sum_{a\in\cA} \EE_{x_s,\chosenset_s,\feedbackset_s} \left\{(f(x_s,a) - f^*(x_s,a))^2 \times \ind{a\in\feedbackset_s} \right\}.
\end{align*}
We also have
\begin{align*}
    Y_s^2 &\leq \frac{1}{k} \sum_{a\in\cA} (f(x_s,a) - f^*(x_s,a))^2(f(x_s,a) + f^*(x_s,a) -2 r_s(a))^2 \times\ind{a\in\feedbackset_s} \\
    & \leq \frac{4}{k}  \sum_{a\in\cA} (f(x_s,a) - f^*(x_s,a))^2 \times\ind{a\in\feedbackset_s} .
\end{align*}
\end{proof}

\begin{lemma}
\label{lem:reg}
Let $\pred_l$ be the estimate of the regression function $f^*$ at epoch $l$. Assume the conditional independence structure in Lemma~\ref{lem:variance} and suppose Assumption~\ref{asum:realizability} holds. Let $\cH_{t-1}$ denote history (filtration) up to time $t-1$. Then for any $\delta<1/e$,
\begin{align*}
\cE = \left\{ l \geq 2: \sum_{s = 1}^{N_{l-1}} \EE_{x_s,\chosenset_s}\left\{ \frac{1}{k}\sum_{a\in\chosenset_s}  (\pred_l(x_s,a) - f^*(x_s,a))^2  \vert \cH_{s-1} \right\}
    \leq c^{-1} 81\log \left( 
                \frac{|\cF| N_{l-1}^3}{\delta}
        \right) 
        \right\}
\end{align*}
holds with probability at least $1 - \delta$.
\end{lemma}
\begin{proof}
Following Lemma~\ref{lem:variance}, let
\begin{align*}
    Y_s(f) = \frac{1}{k} \sum_{a\in \cA} \left((f(x_s,a) - r_s(a))^2 - (f^*(x_s,a) - r_s(a))^2\right)\times\ind{a\in\Phi_s}.
\end{align*}
The argument proceeds as in \citep{agarwal2012contextual}.
Let $\EE_s$ and $\VV_s$ denote the conditional expectation and conditional variance given $\cH_{s-1}$. By Freedman's inequality \citep{high_prob_2008}, for any $t$, with probability at least $1 - \delta' \log t$, we have
\begin{align*}
    \sum_{s=1}^{t}\EE_s[Y_s(f)] - \sum_{s=1}^{t}Y_s(f) \leq 4 \sqrt{\sum_{s=1}^{t} \VV_s[Y_s(f)] \log(1/\delta')} + 2\log(1/\delta')
\end{align*}
Let $X(f) = \sqrt{\sum_{s=1}^{t}\EE_s[Y_s(f)]}$, $Z(f) = \sum_{s=1}^{t}Y_s(f)$ and $C = \sqrt{\log(1/\delta')}$. In view of Lemma~\ref{lem:variance}, with probability at least $1-\delta'\log t$,
\begin{align*}
    &X(f)^2 - Z(f) \leq 8 CX(f) + 2C^2 
\end{align*}
and hence
\begin{align*}   
    (X(f) - 4C)^2 \leq Z(f) + 18C^2 .
\end{align*}
Consequently, with the aforementioned probability, for all functions $f 
\in \cF$ (and, in particular, for $\pred_l$),
\begin{align*}
(X(f) - 4C')^2 \leq Z(f) + 18C'^2 
\end{align*}
where $C' = \sqrt{\log(|\cF|/\delta')}$. Now recall that 
$$\pred_l = \argmin_{f\in\cF} \sum_{t=1}^{N_{l-1}} \sum_{a\in \feedbackset_t} (f(x_{t}, a)-r_{t}(a))^2$$
where $\feedbackset_t$ is a random feedback set satisfying Assumption~\ref{asum:feedback}. Hence, $Z(\pred_l) \leq 0$ for $t=N_{l-1}$, implying that with probability at least $1 -\delta'/(N_{l-1}^2)$,
\begin{align*}
    &\sum_{s = 1}^{N_{l-1}}\frac{1}{k} \sum_{a\in\cA} \EE_{x_s,\chosenset_s,\feedbackset_s} \left\{(\pred_l(x_s,a) - f^*(x_s,a))^2 \times \ind{a\in\feedbackset_s} \vert \cH_{s-1} \right\}
    \leq 81\log\left(\frac{|\cF| N_{l-1}^2\log(N_{l-1})}{\delta'}\right).
\end{align*}
We now take a union bound over $l$ and recall that $\sum_{i\geq 1} 1/i^2 = \pi^2/6 <2$. 

Finally, observe that by Assumption~\ref{asum:feedback},
\begin{align*}
    &\EE_{x_s,\chosenset_s,\feedbackset_s} \left\{(\pred_l(x_s,a) - f^*(x_s,a))^2 \times \ind{a\in\feedbackset_s} \vert \cH_{s-1} \right\} \\
    &= \EE_{x_s,\chosenset_s} \left\{(\pred_l(x_s,a) - f^*(x_s,a))^2 \times \ind{a\in\chosenset_s} \times  \mathbb{P}(a\in\feedbackset_s|x_s,\chosenset_s) \vert \cH_{s-1} \right\} \\
    &\geq c \cdot \EE_{x_s,\chosenset_s} \left\{(\pred_l(x_s,a) - f^*(x_s,a))^2 \ind{a\in\chosenset_s} \vert \cH_{s-1} \right\}.
\end{align*}
We conclude that with probability at least $1-2\delta'$, for all $l\geq 2$,
\begin{align*}
    &\sum_{s = 1}^{N_{l-1}}\frac{1}{k} \EE_{x_s,\chosenset_s}\left\{ \sum_{a\in\chosenset_s}  (\pred_l(x_s,a) - f^*(x_s,a))^2  \vert \cH_{s-1} \right\}
    \leq c^{-1}81\log\left(\frac{|\cF| N_{l-1}^2\log(N_{l-1})}{\delta'}\right).
\end{align*}
\end{proof}

\section{Regression Martingale Bound with Misspecification}
\begin{lemma}
\label{lem:variance2}
Under the notation and assumptions of Lemma~\ref{lem:variance}, but in the case of misspecified model (Assumption~\ref{asum:realizability2} replacing Assumption~\ref{asum:realizability}), it holds that 
\begin{align*}
    \VV_{x_s,\br_s,\chosenset_s,\feedbackset_s} [Y_s] &\leq 8 \EE_{x_s,\br_s,\chosenset_s,\feedbackset_s} [Y_s] + 16\epsilon^2.
\end{align*}
\end{lemma}
\begin{proof}
The proof is along the lines of Lemma~\ref{lem:variance} (see also \citep{foster2020beyond}). We have for any $f:\cX\times\cA\to[0,1]$,
\begin{align*}
    \EE_{x_s,\br_s,\chosenset_s,\feedbackset_s} [Y_s] &= \frac{1}{k}\sum_{a\in\cA}\EE_{x_s,\br_s,\chosenset_s,\feedbackset_s} \left\{ (f(x_s,a) - f^*(x_s,a))(f(x_s,a) + f^*(x_s,a) -2 r_s(a)) \times \ind{a\in\feedbackset_s} \right\} \\
    &= \frac{1}{k} \sum_{a\in\cA} \EE_{x_s,\chosenset_s,\feedbackset_s} \left\{(f(x_s,a) - f^*(x_s,a))^2 \times \ind{a\in\feedbackset_s} \right\} \\
    &+ \frac{2}{k} \sum_{a\in\cA} \EE_{x_s,\chosenset_s,\feedbackset_s} \left\{(f(x_s,a) - f^*(x_s,a))(f^*(x_s,a)-\EE_{\br_s}[r(a)|x_s]) \times \ind{a\in\feedbackset_s} \right\} .
\end{align*}
Rearranging, using AM-GM inequality, and Assumption~\ref{asum:realizability2},
\begin{align*}
    &\frac{1}{k} \sum_{a\in\cA} \EE_{x_s,\chosenset_s,\feedbackset_s} \left\{(f(x_s,a) - f^*(x_s,a))^2 \times \ind{a\in\feedbackset_s} \right\} \\
    &= \EE_{x_s,\br_s,\chosenset_s,\feedbackset_s} [Y_s] - \frac{2}{k} \sum_{a\in\cA} \EE_{x_s,\chosenset_s,\feedbackset_s} \left\{(f(x_s,a) - f^*(x_s,a))(f^*(x_s,a)-\EE_{\br_s}[r(a)|x_s]) \times \ind{a\in\feedbackset_s} \right\} \\
    &\leq \EE_{x_s,\br_s,\chosenset_s,\feedbackset_s} [Y_s] + \frac{1}{2k} \sum_{a\in\cA} \EE_{x_s,\chosenset_s,\feedbackset_s} \left\{(f(x_s,a) - f^*(x_s,a))^2 \times \ind{a\in\feedbackset_s} \right\} + 2\epsilon^2.
\end{align*}
Rearranging,
\begin{align*}
    &\frac{1}{k} \sum_{a\in\cA} \EE_{x_s,\chosenset_s,\feedbackset_s} \left\{(f(x_s,a) - f^*(x_s,a))^2 \times \ind{a\in\feedbackset_s} \right\} \leq 2 \EE_{x_s,\br_s,\chosenset_s,\feedbackset_s} [Y_s] + 4\epsilon^2.
\end{align*}
On the other hand,
\begin{align*}
    Y_s^2 &\leq \frac{1}{k} \sum_{a\in\cA} (f(x_s,a) - f^*(x_s,a))^2(f(x_s,a) + f^*(x_s,a) -2 r_s(a))^2 \times\ind{a\in\feedbackset_s} \\
    & \leq \frac{4}{k}  \sum_{a\in\cA} (f(x_s,a) - f^*(x_s,a))^2 \times\ind{a\in\feedbackset_s} .
\end{align*}
Combining the two inequalities concludes the proof.
\end{proof}

\begin{lemma}
\label{lem:reg2}
Let $\pred_l$ be the estimate of the regression function $f^*$ at epoch $l$. Assume the conditional independence structure in Lemma~\ref{lem:variance} and suppose Assumption~\ref{asum:realizability2} holds. 
Let $\cH_{t-1}$ denote history (filtration) up to time $t-1$. Then for any $\delta<1/e$,
\begin{align*}
\cE = \left\{ l \geq 2: \sum_{s = 1}^{N_{l-1}} \EE_{x_s,\chosenset_s}\left\{ \frac{1}{k}\sum_{a\in\chosenset_s}  (\pred_l(x_s,a) - f^*(x_s,a))^2  \vert \cH_{s-1} \right\}
    \leq c^{-1} 210\log \left( 
                \frac{|\cF| N_{l-1}^3}{\delta} 
        \right) 
        + \epsilon^2 N_{l-1}
        \right\}
\end{align*}
holds with probability at least $1 - \delta$.
\end{lemma}
\begin{proof}
We follow the proof of Lemma~\ref{lem:reg} to see how the misspecification level $\epsilon^2$ enters the bounds.

Let $X(f) = \sum_{s=1}^{t}\EE_s[Y_s(f)]$, $Z(f) = \sum_{s=1}^{t}Y_s(f)$, $C = \log(1/\delta')$ and $M = \epsilon^2 t$. Now using Lemma~\ref{lem:variance2} and Freedman's inequality in the proof of Lemma~\ref{lem:reg}, we find that with probability at least $1-\delta'\log t$,
\begin{align*}
    &X(f) - Z(f) \leq 8 \sqrt{C(2X(f) + 4\epsilon^2 t)} + 2C \\
    &\implies (X(f) - Z(f) - 2C)^2 \leq 128X(f)C + 256MC \\
    &\implies (X(f) - 66C - Z(f))^2 \leq 4352C^2 + 256MC + 128Z(f)C.
\end{align*}
The above bound holds for a fixed function $f$. We now apply an union bound to conclude that for all functions $f 
\in \cF$, with probability at least $1 - \delta'\log t$,
\begin{align*}
    (X(f) - 66C' - Z(f))^2 &\leq 4352C'^2 + 256MC' + 128Z(f)C'  \\
    &\leq 20736 C'^2 + M^2 + 128 Z(f)C'
\end{align*}
where $C' = \log(|\cF|/\delta')$. As in Lemma~\ref{lem:reg}, $Z(\pred_l)\leq 0$ when $t=N_{l-1}$ and thus with probability at least $1 -\delta'\log (N_{l-1})$,
\begin{align*}
    X(\pred_l) &\leq  210C' + \epsilon^2 N_{l-1}.
\end{align*}
Hence, with probability at least $1-\delta'/N_{l-1}^2$,
\begin{align*}
    \sum_{s = 1}^{N_{l-1}}\frac{1}{k} \sum_{a\in\cA} \EE_{x_s,\chosenset_s,\feedbackset_s} \left\{(\pred_l(x_s,a) - f^*(x_s,a))^2 \times \ind{a\in\feedbackset_s} \vert \cH_{s-1} \right\}
    &\leq 210\log\left(\frac{|\cF| N_{l-1}^2\log(N_{l-1})}{\delta'}\right) \\
    &+ \epsilon^2 N_{l-1}.
\end{align*}
The rest of the proof proceeds exactly as in Lemma~\ref{lem:reg}.
\end{proof}


\section{Reduction from eXtreme to $\log(A)$-armed Contextual Bandits}
\label{sec:reduction}
In this section we will prove Corollary~\ref{cor:xtm} which is a reduction style argument. We reduce the $A$ armed top-$k$ contextual bandit problem under Definition~\ref{def:consistent} to a $Z$ armed top-$k$ contextual bandit problem where $Z = O(\log A)$. 

\begin{proof}[Proof of Corollary~\ref{cor:xtm}]
Note that the proof of Theorem~\ref{thm:topkregret} does not require the physical definition of an arm being consistent across all contexts as long as realizability  holds. Let us assume w.l.o.g that Algorithm~\ref{alg:beam} returns the internal and leaf effective arms for any context $x$ in $\cA_x$ in a deterministic ordering. Let us call the $j$-th effective arm in this ordering for any context as arm $j$. This defines a system with $Z$ arms where $Z \leq (p-1)b(H-1) + bm$ as $Z$ is the number of effective arms returned by the beam-search in Algorithm~\ref{alg:beam}. Recall the definition of the new function class $\tilde{\cF}$ from Section~\ref{sec:algo2}. We can thus say that when Definition~\ref{def:consistent} holds this new system is a $Z$ armed top-$k$ contextual bandit system with realizablity (Assumption~\ref{asum:realizability}) with function class $\tilde{\cF}$. Therefore the first part of coroallary~\ref{cor:xtm} is implied by Theorem~\ref{thm:topkregret}. Similarly when Definition~\ref{def:consistent} holds along with Assumption~\ref{asum:realizability2}, this new system is a $Z$ armed top-$k$ contextual bandit system with $\epsilon$-realizablity (Assumption~\ref{asum:realizability2}) with function class $\tilde{\cF}$. Therefore the second part of corollary~\ref{cor:xtm} is implied by Theorem~\ref{thm:topkregret2}. Note that we have used the fact $|\tilde{\cF}|=|\cF|$.
\end{proof}


\section{More Experiments}
\label{sec:morexp}

\begin{figure*}
  \centering
  \subfloat[eurlex-4k]{\label{figur:1}\includegraphics[width=0.4\linewidth]{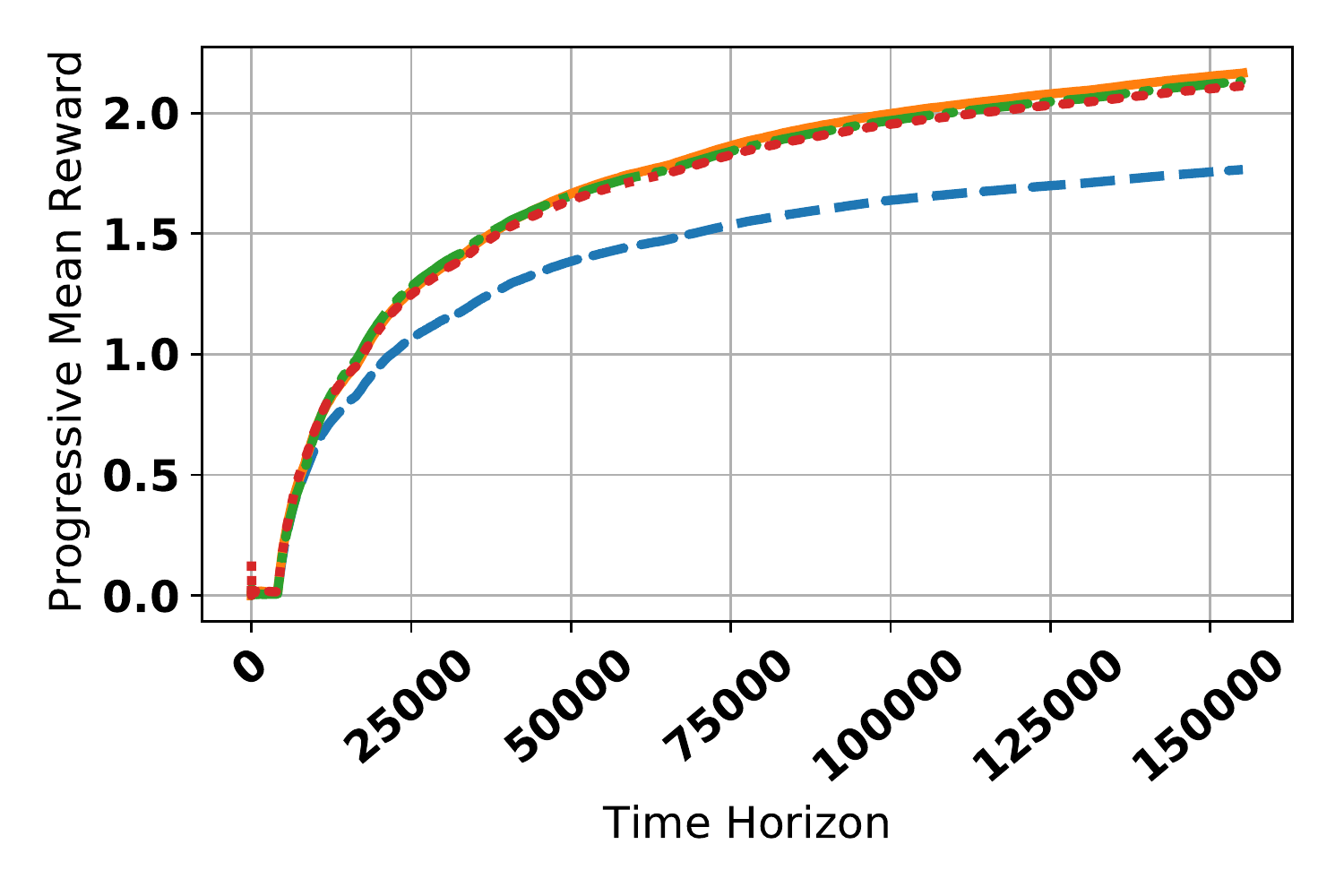}}
  \subfloat[amazoncat-13k]{\label{figur:2}\includegraphics[width=0.4\linewidth]{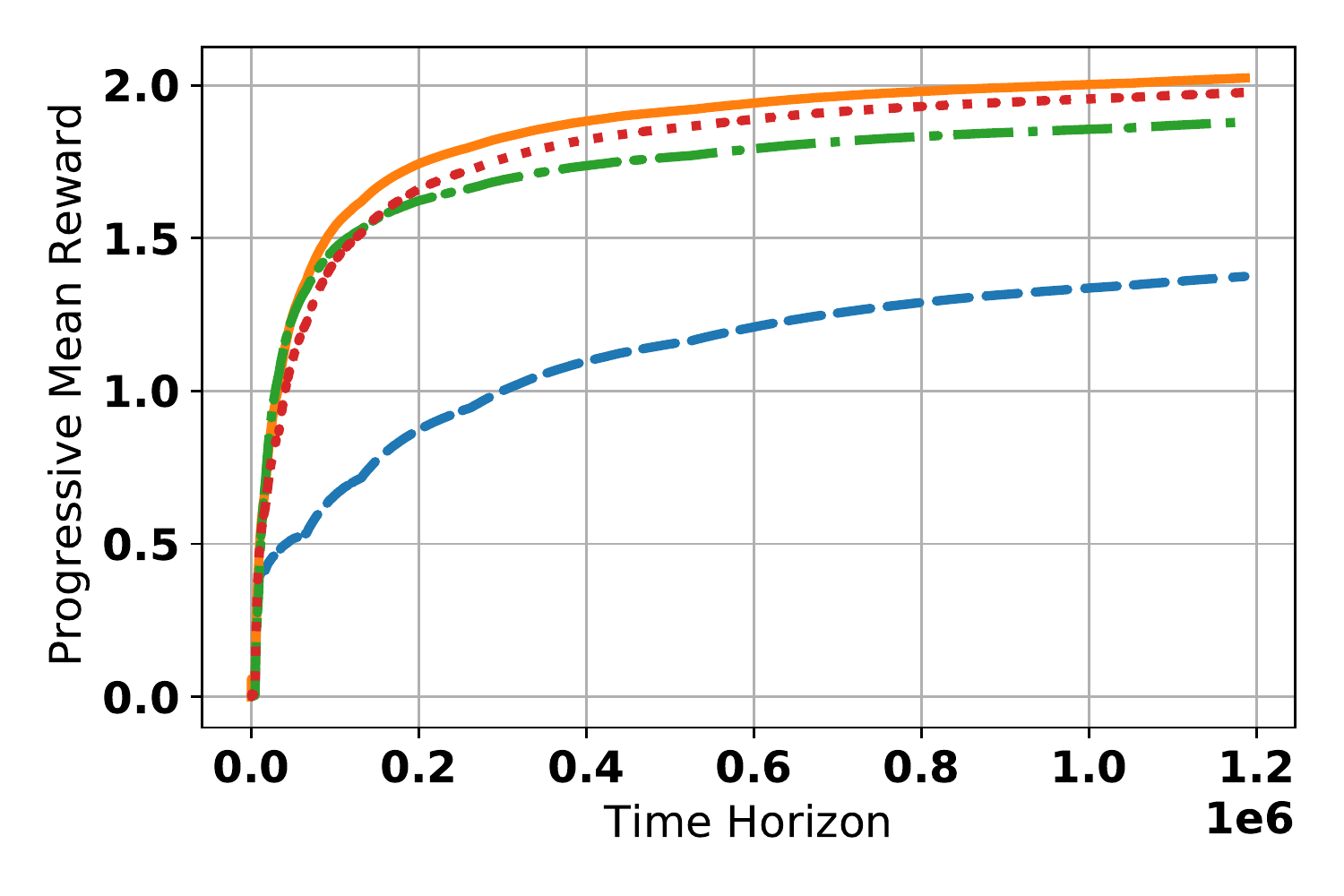}}\\
  \subfloat[wiki10-31k]{\label{figur:3}\includegraphics[width=0.4\linewidth]{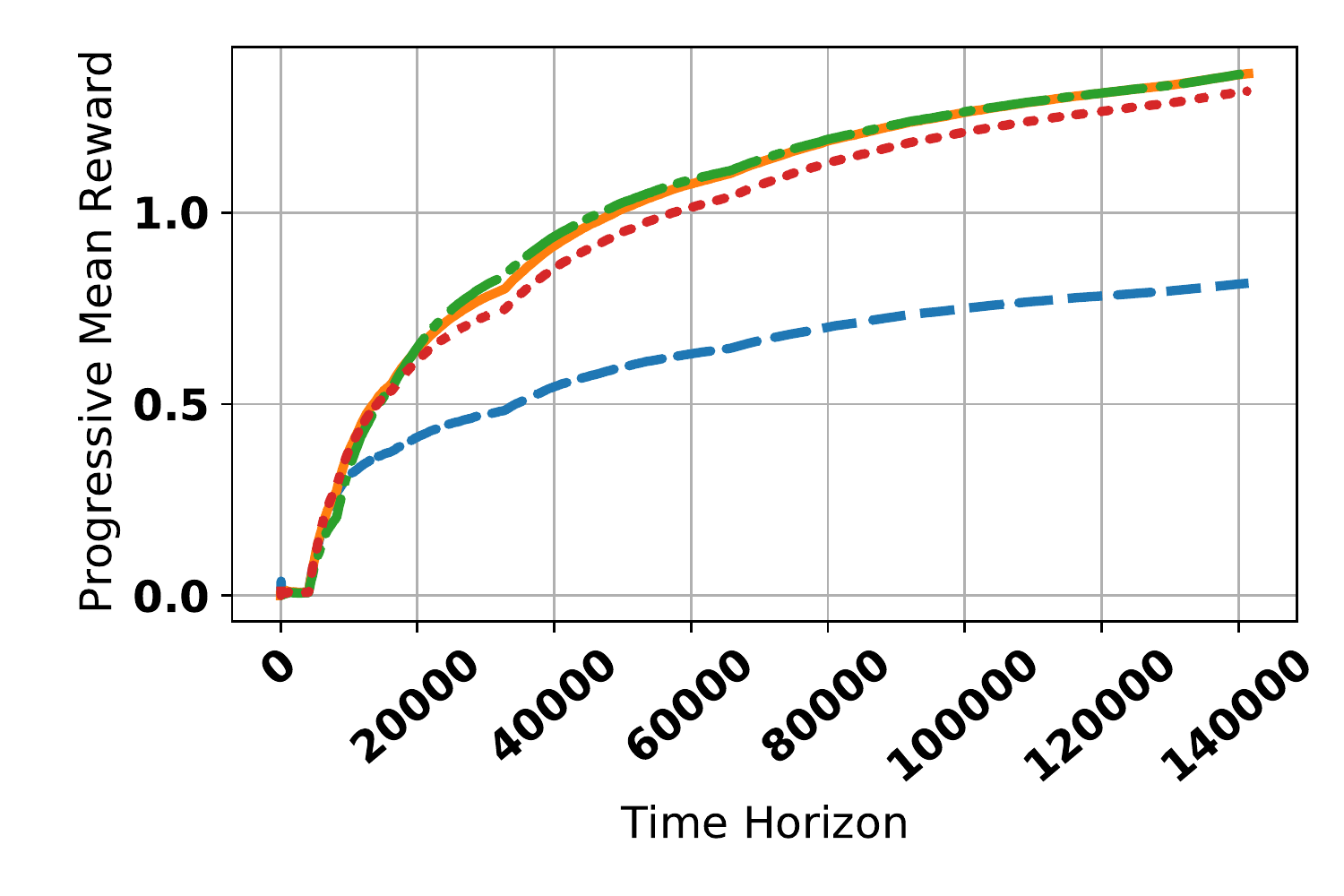}} 
  \subfloat[wiki-500k]{\label{figur:4}\includegraphics[width=0.4\linewidth]{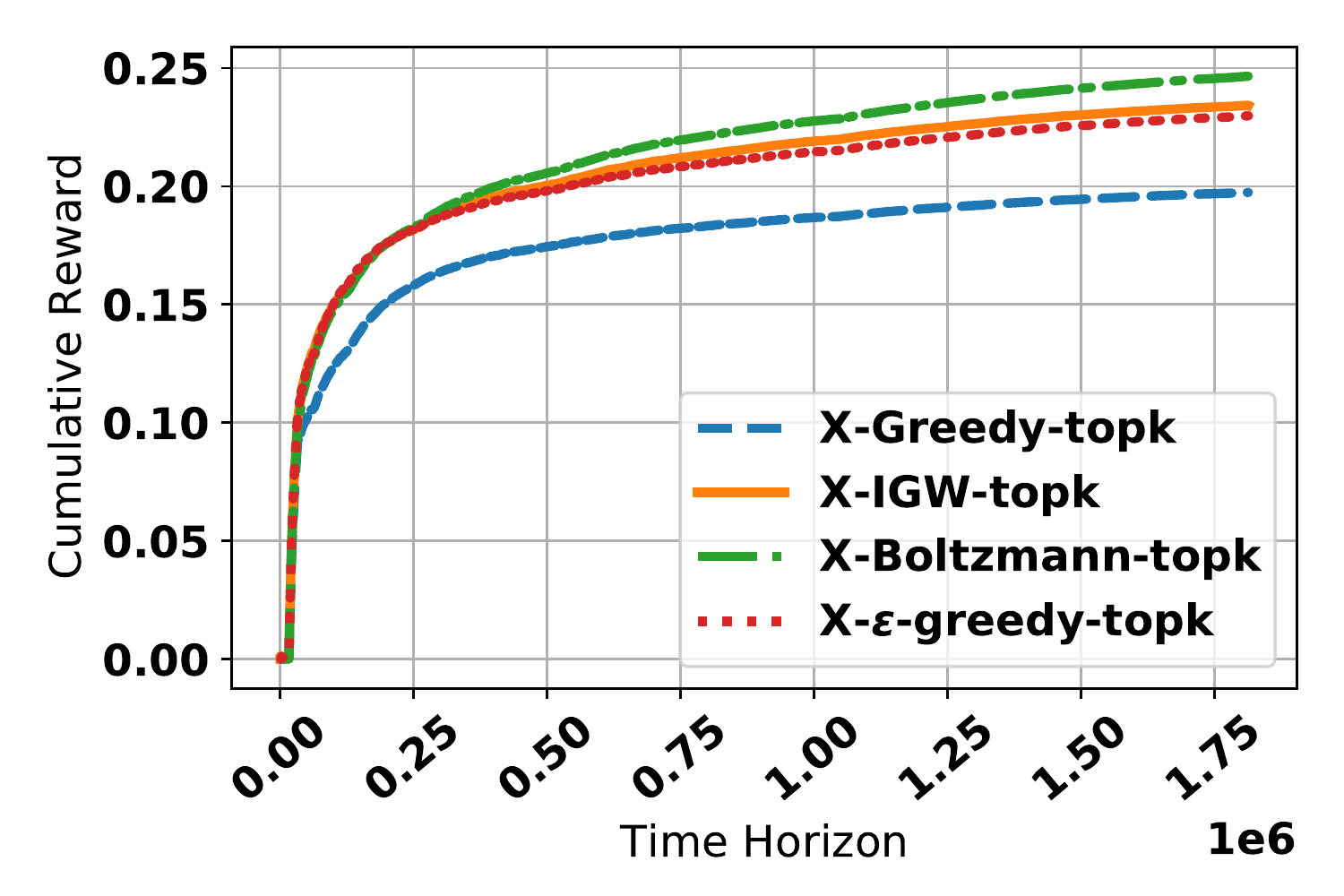}}\\
  \subfloat[amazon-670k]{\label{figur:5}\includegraphics[width=0.4\linewidth]{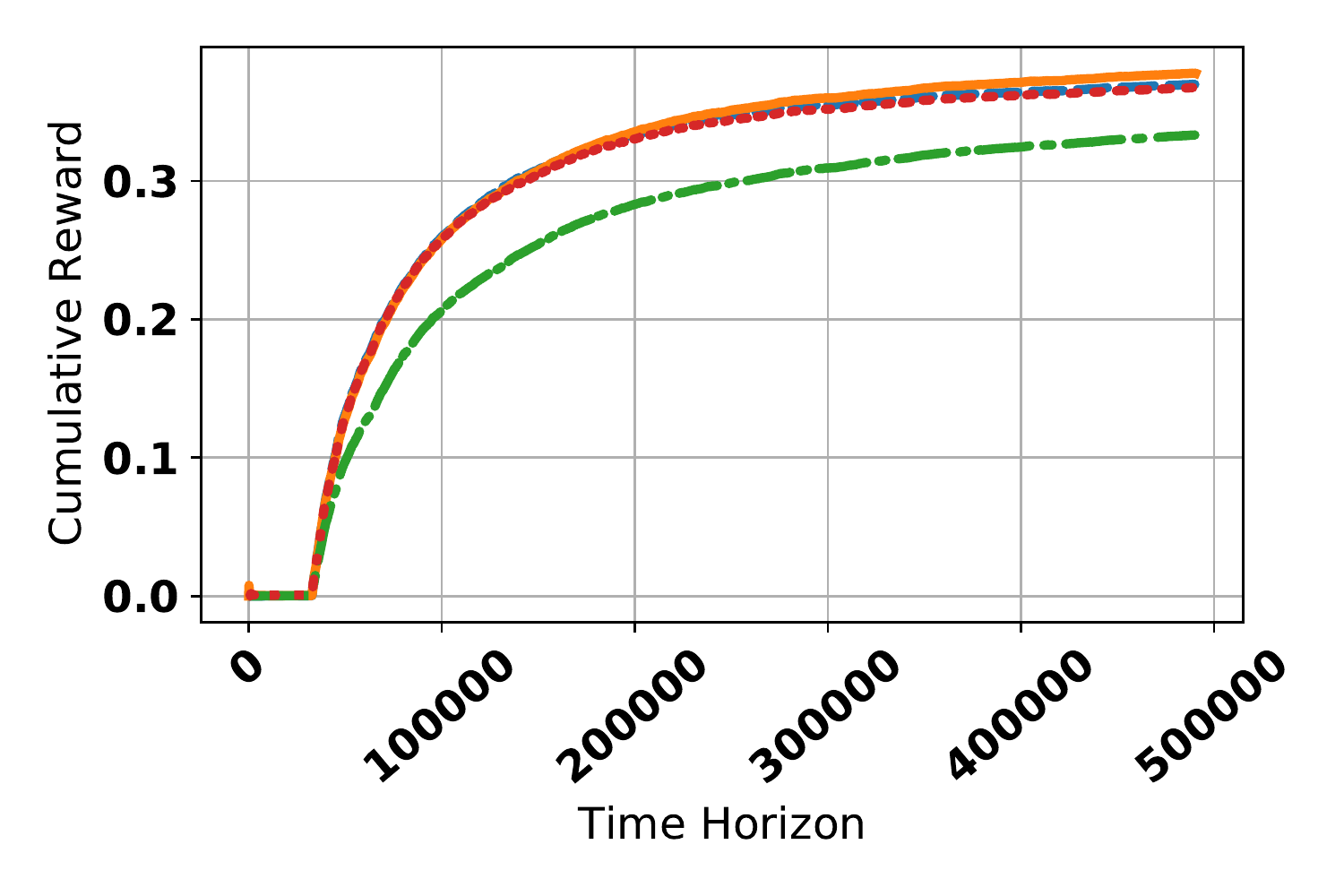}}
  \subfloat[amazon-3m]{\label{figur:6}\includegraphics[width=0.4\linewidth]{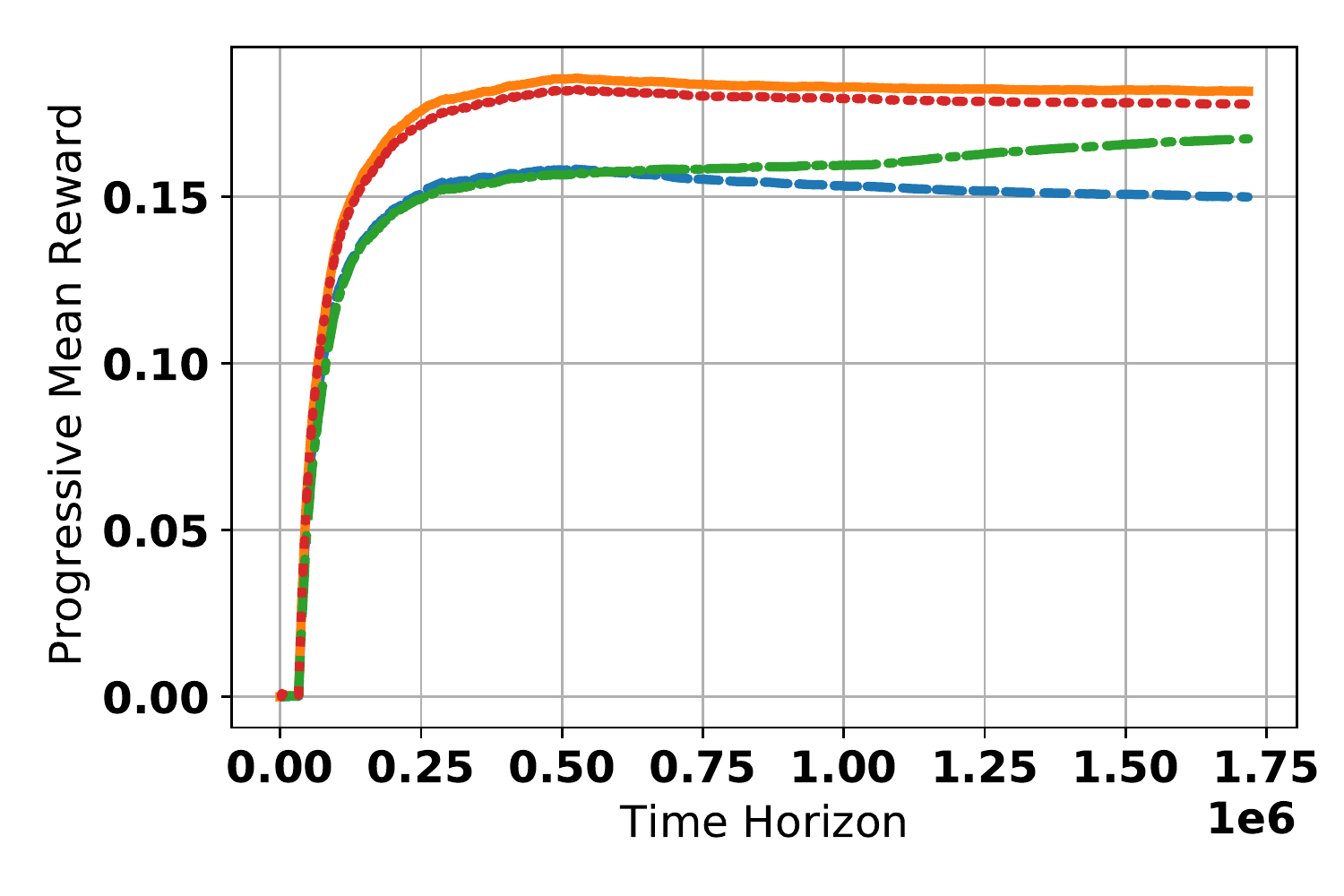}}
 \caption{\small We plot the progressive mean rewards collected by each algorithm as a function of time. All algorithms are implemented under our \xtm reduction framework. The initialization held out set for each dataset is used to train the hierarchy and the routing functions. Then the regressors for all nodes are trained on collected data at the beginning of each epoch. In all our experiments we have $k=5$. In Algorithm~\ref{alg:hier} we set the number of explore slots $r=3$. The common legend for all the plots is provided in (d). The beam-size used is $b=10$.}
 \label{fig:expapp}
\end{figure*}
In Figure~\ref{fig:expapp} we plot the progressive mean rewards vs time for all the experiments using simulated bandit feedback on \xtm datasets.

\section{Implementation Details}
\label{sec:pecos}
For the realizable experiment on Eurlex-4k shown in Figure~\ref{fig:rela}, the optimal weights $\nu^*$'s are obtained by training ridge regression on the rewards vs context for each arm in the dataset. During the experiment we also use the same function class, that is one ridge regression is trained per arm on all collected data during the course of the algorithm. The reward for arm $a$ given context $x$ is chosen as $r_t(a) = [x; 1.0]^T\nu_a^* + \epsilon_t$, where $\epsilon_t$ is a zero-mean Gaussian noise.

{\bf Simulated Bandit Feedback: } A sample in a multi-label dataset can be described as $(x, \mathbf{y})$ where $x \in \cX$ can be thought of as the context while $\mathbf{y} \in \{0, 1\}^{L}$ denotes the correct classes.  We can shuffle such a dataset into an ordering  $\{(x_t, \mathbf{y}^{(t)})\}_{t=1}^{T}$. Then we feed one sample from the dataset at each time step to the contextual bandit algorithm that we are evaluating, in the following manner, 
\begin{compactitem}
\item at time $t$, send the input $x_t$ to the contextual bandit algorithm,
\item the contextual bandit algorithm then chooses an action corresponding to $k$ arms $\ba_t$,
\item the environment then reveals the reward for {\bf only} the $k$ arms chosen $\br_t(\ba_t)$, i.e. whether the arms chosen are among the correct classes or not. 
\end{compactitem}

Note that the algorithm is free to optimize its policy for choosing arms based on everything it has seen so far. In practice however, most contextual bandit algorithms will improve their policy (the $\pred$ it has learnt) in batches. 
The total number of positive classes selected by the algorithm in this process is the total reward collected by the algorithm.

{\bf \xtm Framework:} We follow the framework described in Section~\ref{sec:xtm}. 
We first form the tree and the routing functions from the held out portion of each dataset. The assumption is that there is a small supervised dataset available to each algorithm before proceeding with the simulated bandit feedback experiment. This dataset is used to form a balanced binary tree over the labels till the penultimate level. The nodes in the penultimate level can have a maximum of $m$ children which are the original arms. The value of $m$ is specified in Table~\ref{tab:stats} for each dataset. The division of the labels in each level of the tree is done through hierarchical 2-means clustering over label embeddings, where at each clustering step we use the algorithm from~\citep{dhillon2001co}. The specific label embedding technique that we use is called Positive Instance Feature Aggregation (PIFA) (see ~\citep{prabhu2018parabel} for more details). The routing functions for each internal node in the tree is essentially a one-vs-all linear classifier trained on the held out set. The classifiers are trained using a SVM $\ell_2$-hinge loss. The positive and negative examples for each internal node is selected similar to the strategy in~\citep{prabhu2018parabel}. Finally for the regression function $\ftil(x, \atil)$ where $\atil$ can be an original arm or an internal node in the tree, we train a linear regressor $\ftil(x, \atil) = \nu_{\atil}^T [x; 1]$ as we progress through the experiment as in Algorithm~\ref{alg:hier}. Note that the held out dataset is only used to train the tree and the routing function for each of the algorithms, while the regression functions are trained from scratch only using the samples observed during the bandit feedback experiment.  The details are as follows:

\begin{compactitem}
\item {\bf Tree: } Initially a small part of the dataset is supplied to the algorithms in full-information mode. The size of this portion is captured in Table~\ref{tab:stats} in the Initialization Size column. This portion is used to construct an approximately balanced binary tree over the labels. A supervised multilabel dataset can be represented as $(X, Y)$ where $X \in \mathbb{R}^{n \times d}$ and $Y \in \mathbb{R}^{n \times L}$. We form an embedding for each label using PIFA~\citep{prabhu2018parabel, pecos2020}. Essentially the embedding for each label is the average of all instances that the label is connected to, normalized to $\ell_2$ norm 1. Then we use approximately balanced $2$-means recursively to form the tree until each leaf has less than a predefined maximum number of labels. The exact clustering algorithm used at each step is~\citep{dhillon2001co}.
\item {\bf Routing Functions: } The routing functions are essentially one-vs-all linear classifiers at each internal node of the tree. The positive examples for the classifier at an internal node are the input instances in the small supervised dataset that have a positive label in the subtree of that node. The negative instances are the set of all instances that has a positive label in the subtree of the parent of that node but not in that node's subtree. This is the same methodology as in~\citep{prabhu2018parabel}. The routing functions are trained using LinearSVC~\citep{fan2008liblinear}.
\item{\bf Regression Functions: } After creating the tree and the routing function from the small held out set, they are held fixed. The function class $\tilde{\cF}$ as Algorithm~\ref{alg:hier} progresses is a set of linear regression functions at each internal and leaf node of the tree. They are trained on past data collected during the course of the previous epochs. Note that the examples for training the regressor for an internal node are only from the singleton arms that were shown when the algorithm selected that particular internal node in the \igw sampling. The regression functions are trained using LinearSVR~\citep{fan2008liblinear}.
\item{\bf Hyper-parameter Tuning: } For all the exploration algorithms in the \xtm experiments the parameters are tuned over the eurlex-4k dataset and then held fixed. For the \igw scheme $C$ is tuned over a grid of $\{1e-7, 1e-6, \cdots, 1e7\}$. The same is done for the $\beta$ in the Boltzmann scheme. For $\epsilon$-greedy the $\epsilon$ value is tuned between $[1e-7, 1.0]$ in a equally spaced grid in the logarithmic scale. The best parameters that are found are $\beta=1.0, C=1.0$ and $\epsilon=0.167$.
\item{\bf Inference:} Inference using a trained model is done exactly according to Algorithm~\ref{alg:hier}. The beam-search over the routing function yields effective arms. Then we evaluate the linear regression functions for each of the effective arms (singleton arms or the internal nodes in the tree). If a non-singleton effective arm is chosen among the $k$ arms we randomly sample a singleton arm in it's subtree. The beam search and \igw sampling is implemented in C++ where the linear operations are implemented using the Eigen package~\citep{eigenweb}.
\end{compactitem}

\end{document}